\RequirePackage[loading]{tracefnt}
\documentclass[sigconf]{acmart}

\usepackage{booktabs} 
\usepackage{bm}
\usepackage{dsfont}
\usepackage{tikz}
\usetikzlibrary{shapes.geometric,arrows}
\usepackage{subcaption}
\usepackage{verbatim}

\newcommand{\jl}[1]{\comments{\textcolor{red}{[Joel: #1]}}}

\newcommand{\removed}[1]{#1}
\newcommand{\removedJoel}[1]{#1}
\newcommand{\add}[1]{#1}
\newcommand{\alexchange}[1]{#1}
\newcommand{\change}[1]{#1}

\newcommand{\beginsupplement}{%
        \setcounter{table}{0}
        \renewcommand{\thetable}{S\arabic{table}}%
        \setcounter{figure}{0}
        \renewcommand{\thefigure}{S\arabic{figure}}%
        \setcounter{section}{0}
        \renewcommand{\thesection}{S\arabic{section}}%
     }

\setcopyright{rightsretained}

\acmDOI{10.1145/nnnnnnn.nnnnnnn}

\acmISBN{978-x-xxxx-xxxx-x/YY/MM}

\copyrightyear{2019} 
\acmYear{2019} 
\setcopyright{acmlicensed}
\acmConference[GECCO '19]{Genetic and Evolutionary Computation Conference}{July 13--17, 2019}{Prague, Czech Republic}
\acmBooktitle{Genetic and Evolutionary Computation Conference (GECCO '19), July 13--17, 2019, Prague, Czech Republic}
\acmPrice{15.00}
\acmDOI{10.1145/3321707.3321876}
\acmISBN{978-1-4503-6111-8/19/07}

\begin{document}
\title{Evolvability ES: Scalable and Direct Optimization of Evolvability}

\author{Alexander Gajewski}
\authornote{Work done during an internship at Uber AI Labs}
\affiliation{Columbia University}
\author{Jeff Clune}
\affiliation{Uber AI Labs}
\affiliation{University of Wyoming}
\author{Kenneth O. Stanley}
\affiliation{Uber AI Labs}
\author{Joel Lehman}
\affiliation{Uber AI Labs}
\begin{abstract}
Designing \change{evolutionary} algorithms capable of uncovering highly \emph{evolvable} representations is an
open challenge \removedJoel{in evolutionary computation}; such evolvability is important \removedJoel{in practice,} because it \change{accelerates} evolution and enables fast adaptation to changing circumstances. This paper introduces \emph{evolvability ES}, an evolutionary algorithm designed to explicitly and efficiently optimize for evolvability, i.e.\ the ability to further adapt. The insight is that \removedJoel{it is possible to} derive a novel objective in the spirit of natural evolution strategies that maximizes the diversity of behaviors exhibited when an individual is subject to random mutations, and that efficiently scales with computation. Experiments in 2-D and 3-D locomotion tasks highlight the potential of evolvability ES to generate solutions with tens of thousands of parameters that can quickly be adapted to solve different tasks and that can productively seed further evolution. We further highlight a connection between evolvability \removedJoel{in EC} and a recent and popular gradient-based meta-learning algorithm called MAML; results show that evolvability ES can perform competitively with MAML and that it discovers solutions with distinct properties. The conclusion is that evolvability ES opens up novel research directions for studying and exploiting the potential of evolvable representations for deep neural networks.
\end{abstract}



\maketitle

\vspace{-0.05in}
\section{Introduction}


One challenge in evolutionary computation (EC) is to design algorithms capable of uncovering highly \emph{evolvable} representations; \change{though evolvability's definition is debated, the idea is to find genomes with great potential for further evolution} \cite{altenberg:evolution,wagner1996perspective,reisinger:empirical,mengistu_2016,lehman2011improving,ebner:neutral,grefenstette:evolvability,kounios2016resolving}. 
\change{Here, as in previous work, we adopt a definition of evolvability as the propensity of an individual to generate phenotypic diversity \cite{mengistu_2016,lehman2011improving,lehman2013evolvability}}.
Such evolvability is important \removedJoel{in practice,} because it \change{\alexchange{broadens the variation} accessible through mutation, thereby accelerating evolution}; improved \change{evolvability} thus would benefit many areas across EC, e.g.\ evolutionary robotics, open-ended evolution, and quality diversity \change{(QD; \cite{pugh:quality,lehman:vc})}.
While evolvability is seemingly ubiquitous in nature \removedJoel{(e.g.\ the amazing diversity of dogs accessible within a few generations of breeding)},
its emergence in evolutionary algorithms (EAs) is seemingly rare \cite{wagner1996perspective,reisinger:empirical}, and how
best to encourage it remains an important open question. 

There are two general approaches to encourage evolvability in EAs. The first is to create environments or
selection criteria that produce evolvability \change{as an \emph{indirect}} consequence \cite{lehman2011improving,ebner:neutral,reisinger:empirical,kashtan2007varying,clune2013evolutionary}. For example, 
environments wherein goals vary modularly over generations may implictly favor individuals better able to adapt to such variations \cite{kashtan2007varying}. \change{The second approach\alexchange{,} which is the focus of this paper\alexchange{,} is to select \emph{directly} for evolvability, \alexchange{i.e.}\ to judge \alexchange{individuals} by directly testing \alexchange{their} potential for further evolution \cite{mengistu_2016}.}
While the first approach is more biologically plausible and is important to understanding
natural evolvability, the second benefits from its directness, its potential ease of application to new domains, and its ability to enable the \emph{study} of highly-evolvable genomes without fully understanding evolvability's natural emergence. 
However, current implementations of such \emph{evolvability search} \cite{mengistu_2016} suffer from their computational cost.

A separate (but complementary) challenge in EC is that of effectively evolving \emph{large} genomes. For example, there has been recent interest in training deep neural networks (DNNs)
because of their potential for expressing complex behaviors, e.g.\ playing Atari games from raw pixels \cite{mnih2015human}. 
However, evolving DNNs is challenging because they have many more parameters
than genomes typically evolved by comparable approaches in EC (e.g.\ neural networks evolved by NEAT \cite{stanley:ec02}). 
For this reason, the study of scalable EAs that can benefit from increased computation is of recent interest \cite{salimans_2017,such2017deep,backtobasics}, and evolvability-seeking algorithms will require similar considerations to scale effectively to large networks.

This paper addresses both of these challenges, by synthesizing three threads of research in deep
learning and EC. The first thread involves a popular gradient-based meta-learning algorithm called MAML \cite{finn2017model} that
searches for points in the search space from which one (or a few) \change{optimization step(s)} can solve diverse tasks. We introduce here a connection between this kind of \removedJoel{parameter-space} meta-learning and evolvability, as MAML's formulation is very similar to \change{that of} 
evolvability search \cite{mengistu_2016}, \change{which} searches for individuals from which \emph{mutations} (instead of optimization) yield a diverse repertoire of behaviors. MAML's formulation, and its success with DNNs on complicated reinforcement learning \change{(RL)} tasks, hints that there may similarly be efficient and effective formulations of evolvability.
The second thread involves the recent scalable form of evolution strategy (ES) of \citet{salimans_2017} (which at heart is a simplified form of natural evolution strategy \cite{wierstra_2011}) shown to be surprisingly competitive with gradient-based \change{RL}. We refer to this specific algorithm as ES in this paper for simplicity, and note that the field of ES as a whole encompasses \change{many diverse} algorithms \cite{schwefel:es,beyer:nc02}.
The final thread is a recent formalism called stochastic computation graphs \change{(SCGs)} \cite{schulman_2015}, which enables automatic derivations of gradient estimations that include expectations over distributions (such as the objective optimized by ES). We here extend \change{SCGs} to handle a larger class of functions, which enables formulating an efficient evolvability-inspired objective.

Weaving together these three threads, the main insight in this paper is that it is possible to derive  a novel algorithm, called \emph{evolvability ES}, that optimizes an evolvability-inspired objective without incurring \emph{any additional overhead} in domain evaluations relative to optimizing a traditional objective with ES. 
Such efficiency is possible because each iteration of ES \change{can aggregate information \emph{across} samples to estimate local gradients of evolvability}.

The experiments in this paper demonstrate the potential
of evolvability ES in two deep RL \removedJoel{locomotion} benchmarks, wherein
evolvability ES \removedJoel{often} uncovers a diversity of high-performing behaviors using the same computation required for ES to uncover a single one. Further tests highlight the potential benefits of such evolvable genomes for fast adaptability, \change{show that evolvability ES can optimize a form of population-level evolvability \cite{wilder2015reconciling}, and demonstrate that evolvability ES performs competitively with the popular MAML gradient-based method (while discovering pockets of the search space with interestingly distinct properties).}
The conclusion is that evolvability ES is a promising new algorithm for producing and studying evolvability at deep-learning scale.



\vspace{-0.2cm}
\section{Background}

\subsection{Evolvability and Evolvability Search}

No consensus exists on the measure or definition of evolvability \cite{pigliucci:evolvability}; the 
definition we adopt here, as in prior work \cite{mengistu_2016,lehman2011improving,lehman2013evolvability}, follows one mainstream conception of evolvability as phenotypic variability \cite{kirschner1998evolvability,brookfield:evolvability,pigliucci:evolvability}, i.e.\ the phenotypic diversity
demonstrated by an individual's offspring. 
\change{Exact definition aside},
reproducing the evolvability of natural organisms is an important yet unmet goal in EC.
Accordingly, researchers have proposed many approaches for encouraging evolvability \cite{lehman2011improving,kashtan2007varying,ebner:neutral,reisinger:empirical,clune2013evolutionary,nguyen:innovation,kounios2016resolving}.
\change{This paper focuses in particular on extending the ideas from evolvability search \cite{mengistu_2016}}, an algorithm that \emph{directly} rewards evolvability \change{\alexchange{by} explicitly measur\alexchange{ing} evolvability and guid\alexchange{ing} search \alexchange{towards} it}.
The motivation is that it often may be more straightforward to directly optimize evolvability than to encourage its indirect emergence, that it may help compare the advantages or drawbacks of different quantifications of evolvability \cite{lehman2018potential}, and that it may facilitate the study of evolvability even before its natural emergence is understood.

In evolvability search, the central idea is to calculate an individual's fitness from domain evaluations of many of its potential offspring. In particular, an individual's potential to generate phenotypic variability is estimated by quantifying the diversity of behaviors demonstrated from evaluating a sample of its offspring. 
\change{Note the distinction between evolvability search and QD: QD attempts to build a collection of diverse well-adapted genomes, while evolvability search attempts to find genomes from which diverse behaviors can readily be evolved.}
In practice, \change{evolvability search} requires (1) quantifying dimensions of behavior of interest, as in the behavior characterizations of novelty search \cite{lehman2011abandoning}, and (2) a distance threshold to formalize what qualifies as two behaviors being distinct. 
Interestingly, optimizing \change{evolvability}, \change{just like optimizing novelty, can sometimes lead to solving problems as a byproduct \cite{mengistu_2016}.}
However, evolvability search is computationally expensive (\change{it requires evaluating enough offspring of each individual in the population to estimate its potential}), and has only been demonstrated \change{with} small neural networks (NNs); evolvability ES address\change{es} both issues.

\vspace{-0.2cm}

\subsection{MAML}

\emph{Meta-learning} \cite{metalearning} focuses on optimizing
an agent's learning potential (i.e.\ its ability to solve new tasks) rather than its immediate performance (i.e.\ how well it solves the current task) as is more typical in optimization and RL\removedJoel{, and  has a rich history both in EC and machine learning at large}. The most common forms of neural meta-learning algorithms train NNs
to implement their own learning algorithms, either through exploiting recurrence \cite{stanley2003evolving,wang2016learning,hochreiter2001learning} or plastic connections \cite{stanley2003evolving,floreano1996evolution,soltoggio2008evolutionary,miconi2018differentiable}. 

However, another approach is taken by the recent and popular model-agnostic meta-learning (MAML) algorithm \cite{finn2017model}. Instead of training a NN that itself can learn from experience, MAML searches for a \emph{fixed} set of NN weights from which a single additional \change{optimization} step can solve a variety of tasks. The insight is that it is possible to differentiate through the optimization step, to create a fully gradient-based algorithm that seeks adaptable areas of the search space.
Interestingly, MAML's formulation shares deep similarities with the idea of evolvability \removedJoel{in EC}: Both involve finding points in the search space nearby to diverse functional behaviors. 
One contribution here is to expose this connection \removedJoel{between MAML and evolvability} and explore it experimentally.

Intriguingly, MAML has been successful in training DNNs with many parameters to quickly adapt in supervised and RL domains \cite{finn2017model}.
These results suggest that evolvable DNNs exist and can sometimes be directly optimized towards, a key source of inspiration for evolvability ES.

\vspace{-0.1cm}
\subsection{Natural Evolution Strategies}

The evolvability ES algorithm introduced here builds upon the ES algorithm of \citet{salimans_2017}, which itself is based on the Natural Evolution Strategies (NES; \cite{wierstra_2011}) black-box optimization algorithm. Because in such a setting the gradients of the function to be optimized, $f(z)$, are unavailable, NES instead creates a smoother version of $f$ that \textit{is} differentiable, by defining a \change{population} distribution $\pi(z; \theta)$ and setting the smoothed loss function $J(\theta) = \mathbb{E}_{z \sim \pi} \left [f(z) \right]$. This function is then optimized iteratively with gradient descent, where the gradients are estimated by samples from $\pi$.

\citet{salimans_2017} showed recently that NES with an isotropic Gaussian distribution of fixed variance (i.e. $\pi = \mathcal{N}(\mu, \sigma I)$ and $\theta = \mu$) is competitive with deep RL on high-dimensional RL tasks, and can be very time-efficient because the expensive gradient estimation is easily parallelizable. \citet{salimans_2017} refer to this algorithm as Evolution Strategies (ES), and we adopt their terminology in this work, referring to it in the experiments as standard ES \removedJoel{(to differentiate it from evolvability ES)}.

Connecting to traditional EAs, the distribution $\pi$  can be viewed as a \emph{population} of individuals that evolves by one \emph{generation} at every gradient step, where $\mu$ can be viewed as the parent individual, and samples from $\pi$ can be termed that parent's \emph{pseudo-offspring} cloud.
\change{NES is formulated differently from other EAs; unlike them it operates by \emph{gradient descent}}. \change{Because} of NES's formulation, the desired outcome (i.e.\ to increase the average fitness $f$) can be represented by a \emph{differentiable function} of \change{trainable} parameters of the population \change{distribution}. We can then analytically compute the optimal update rule for each generation: to take a step in the direction of the gradient. For such algorithms, the burden \change{shifts} from designing a generational update rule to designing an objective function. Thus, to enable fast experimentation with new objective functions, ES-like algorithms can benefit from \emph{automatic differentiation}, an idea explored in the next two sections.

\vspace{-0.2cm}
\subsection{Stochastic Computation Graphs}


In most modern tensor computation frameworks (e.g.\ PyTorch \cite{tensorflow2015} and Tensorflow \cite{paszke2017automatic}), functions are represented as \emph{computation graphs}, with input nodes representing a function's arguments, and inner nodes representing compositions of certain elementary functions like addition or exponentiation.
The significance of computation graphs is that they support automatic differentiation, enabling easy use of gradient descent-like algorithms with novel architectures and loss functions.
Note, however, that these computation graphs are \emph{deterministic}, as these variables have no probability distributions associated with them.
Recently introduced in \citet{schulman_2015}, Stochastic Computation Graphs (SCGs) are a formalism adding support for random variables.
Importantly, like deterministic computation graphs, SCGs support automatic differentiation.

SCGs are thus a natural way to represent ES-like algorithms: The function ES tries to optimize is an expectation over individuals in the population of some function of their parameters. 
In this way, automatic differentiation of SCGs enables rapid experimentation with modifications of ES to different loss functions or population distributions; more details about SCGs can be found in supplemental sections \ref{sec:computation_graphs} and \ref{sec:nested_computation_graphs}. The next
section extends SCGs, which is needed to explore certain formulations of evolvability ES.

\vspace{-0.2cm}
\section{Nested SCGs}
While SCGs are adequate for simple modifications of ES (e.g.\ changing the population distribution from an isotropic Gaussian to \change{a Gaussian with diagonal covariance}), they lack the expressiveness needed to enable certain more complex objectives.
For example, consider an ES-like objective of the form $\mathbb{E}_z \left[ f(\mathbb{E}_{z'}[g(z, z')]) \right]$ (the form of one variant of evolvability ES).
There is no natural way to represent this as a SCG, because it involves a function applied to a \emph{nested} expectation (the key reason being that in general, $\mathbb{E}_z[f(z)] \ne f(\mathbb{E}_z[z])$).

\removedJoel{In light of these limitations,} in supplemental sections \ref{sec:computation_graphs} and \ref{sec:nested_computation_graphs}, we generalize SCGs \cite{schulman_2015} to \removedJoel{more easily} account for nested expectations, and show that this new formalism yields \emph{surrogate loss functions} that can automatically be differentiated by popular tools like PyTorch \cite{paszke2017automatic}.
Through this approach, nearly any loss function involving potentially nested expectations over differentiable probability distributions can be automatically optimized with gradient descent through sampling. This approach is applied in the next section to enable an efficient evolvability-optimizing variant of ES.

\vspace{-0.2cm}
\section{Approach: Evolvability ES}
One motivation for evolvability ES is to create an algorithm similar in effect to evolvability search \cite{mengistu_2016} but without immense computational requirements, such that
it can scale efficiently with computation to large DNNs.
To enable such an algorithm, a driving insight is that the same domain evaluations exploited in standard ES to estimate gradients of \emph{fitness} improvement also contain  the information needed to estimate gradients of \emph{evolvability} improvement. That is, ES estimates fitness improvement by sampling
and evaluating policies drawn from the neighborhood of a central individual, and updates the central individual such that future draws of the higher-fitness policies are more likely. Similarly, the central individual
could be updated to instead make policies that contribute more to the diversity of the entire population-cloud more likely.
The result is an algorithm that does not incur the cost of evolvability search of evaluating potential offspring of all individuals, because
in an algorithm like ES that represents the population as a formal distribution over the search space, information about local diversity is shared \emph{between individuals}; a further
benefit is the computational efficiency and scalability of ES itself, which this algorithm retains.

While such an insight
sounds similar to what motivates novelty-driven ES algorithms \cite{conti2018neurips,cuccu2011novelty}, there is a subtle but important difference: Novelty-driven ES drives 
the central individual of ES towards behaviors unseen in previous generations, while evolvability ES instead aims to optimize the diversity of offspring of the current central individual. Thus the product of evolvability ES is an evolvable central individual whose mutations lead to diverse behaviors (note that an experiment in section \ref{sec:mixtures} extends evolvability ES to include multiple central individuals).
We now formally describe the method, and two concrete variants\footnote{\change{Source code available at: \url{https://github.com/uber-research/Evolvability-ES}}}.


\subsection{Formal Description}

\setlength{\abovedisplayskip}{3pt}
\setlength{\belowdisplayskip}{3pt}

Consider the isotropic Gaussian distribution $\pi$ of ES. Recall the analogy wherein this distribution represents the population, and the distribution mean can be seen
as the central individual\removedJoel{, or parent}.
By the definition adopted here, \emph{evolvable} points in the search space can through mutation produce many different behaviors: In other words, mutations provide many different options for natural selection to select from.
Thus, as in evolvability search \cite{mengistu_2016}, our aim is to maximize some statistic of behavioral diversity of an individual's mutations.
Formally, behavior is represented as a behavior characteristic (BC; \cite{lehman2011abandoning}), a vector-valued function mapping a genome $z$ to behaviors $\textbf{B}(z)$.
For example, in a locomotion task, a policy's behavior can be its final position on a plane.

Here, we consider two different diversity statistics which lead to maximum variance (MaxVar) and maximum entropy (MaxEnt) variants of evolvability ES, here called MaxVar-EES and MaxEnt-EES. MaxVar-EES maximizes the trace of the covariance matrix of the BC over the population.
The motivation is that the variance of a distribution captures one facet of the distribution's diversity. This formulation results in the following loss function:
\begin{equation} \label{eq:max_var}
    J(\theta) = \sum_j \mathbb{E}_z \left [(B_j(z) - \mu_j)^2 \right ],
\end{equation} 
where the expectation is over policies $z \sim \pi(\cdot ; \theta)$, the summation is over components $j$ of the BC, and $\mu_j$ represents the mean of the $j$th component of the BC.

MaxEnt-EES maximizes a distribution's entropy rather than its variance. We expect \removedJoel{in general} this method to behave more similarly to evolvability search than MaxVar-EES, because \removedJoel{over a fixed, bounded region of support in behavior space,} the maximum-entropy distribution of behavior is a uniform distribution, which also would maximize evolvability search's fitness criterion, i.e.\ to maximize the amount of distinct behaviors displayed by offspring.
To estimate entropy, we first compute a kernel-density estimate of the distribution of behavior for a chosen kernel function $\varphi$, giving
\begin{equation}
    p(\bm{B}(z); \theta) \approx \mathbb{E}_{z'}[\varphi(\bm{B}(z')-z)],
\end{equation}
which can be applied to derive a loss function which estimates the entropy:
\begin{equation} \label{eq:ent_loss}
    J(\theta) = -\mathbb{E}_z \left [ \log \mathbb{E}_{z'}[\varphi(\bm{B}(z')-z)] \right ].
\end{equation}
In practice, these losses are differentiated with PyTorch (MaxEnt-EES in particular depends on the nested SCGs described earlier) and both the loss and their gradients are then estimated from samples \change{(recall that these samples all come from the current population distribution, and information from all samples is integrated to estimate the direction of improving evolvability \alexchange{for the population as a whole})}. The gradient estimates are then applied to update the population distribution, enabling a generation of evolution. 


\vspace{-0.2cm}
\subsection{Proof of Feasibility}

It may at first seem unreasonable to expect that Evolvability ES could even in theory optimize the objective set out for it, i.e.\ to uncover a particular region of DNN weight-space such that mutations can produce the arbitrarily different mappings from input to output necessary to demonstrate diverse behaviors. In addition to the empirical results in this paper, we also prove by construction that such parts of the search space do exist.

\alexchange{Specifically}, we show that given \emph{any} two continuous policies, it is possible to construct a five-layer network from which random mutations results in ``flipping'' between the two initial policies.
This construction is \alexchange{described} in supplemental section \ref{sec:theory}.

\section{Experiments}

\subsection{Interference Pattern Task} \label{sec:interference_task}

To help validate and aid understanding of \removedJoel{the} evolvability ES \removedJoel{approach}, 
consider the interference pattern in figure \ref{fig:interference_plot}a, generated as the product of a high-frequency and a low-frequency sine wave.
In this figure, the horizontal axis represents a single-dimensional parameter space, and the vertical axis represents the single-dimensional BC.
\change{In} this task, evolvability is maximized at the maximum of the slower sine wave, because this is the point at which the amplitude of the faster sine wave is greatest, meaning that mutation will induce the widest distribution of vertical change.
Figure \ref{fig:interference_plot} shows that both variants of evolvability ES approached the optimal value in parameter space as training progressed.
The next sections demonstrate that the algorithm also successfully scales to more complex domains.

\begin{figure*}
\centering
\includegraphics[width=\linewidth]{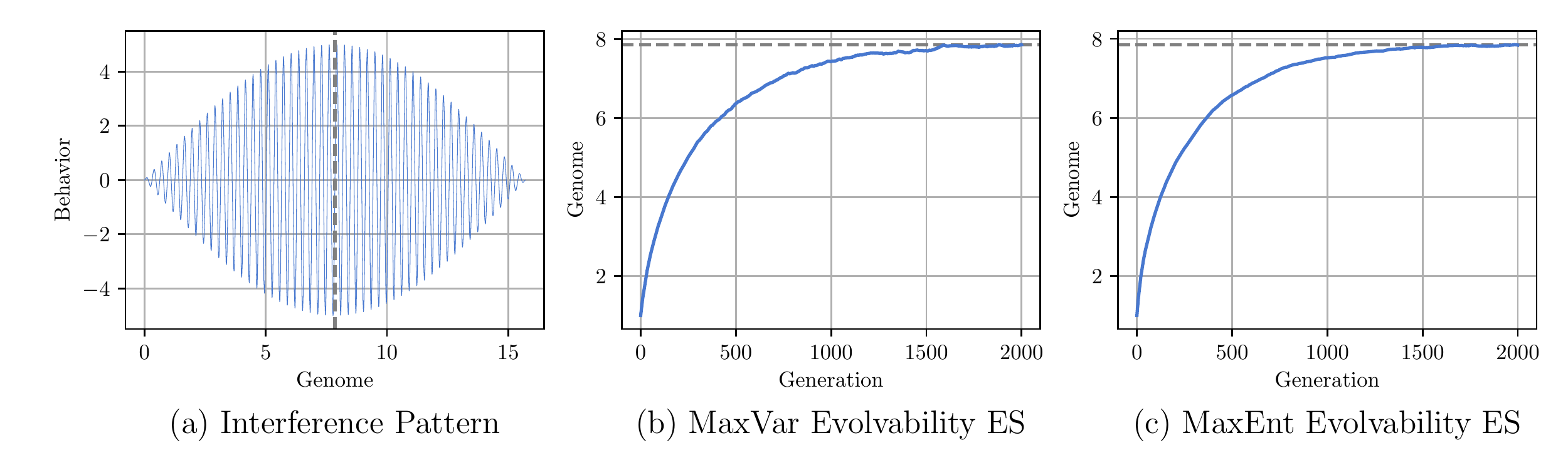}
\vspace{-0.35in}
\caption{Interference Pattern Results. In the interference pattern task, a genome consisted of a single floating point parameter ($x$), and the resulting behavior was generated as a function of the genome parameter by: $f(x) = 5\sin(x/5) \sin(20x)$. (a) Shows the plot of behavior (vertical axis) as a function of genome (horizontal axis). The training plots shown in (b) and (c) validate both evolvablity ES variants, showing that they converged to the point with behavior that is most sensitive to perturbations, shown as a dashed line in all three plots.}
\label{fig:interference_plot}
\end{figure*}

\vspace{-0.2cm}
\subsection{2-D Locomotion Task}

We next applied evolvability ES to a physically-simulated robotics domain. 
In this 2-D locomotion task, a NN policy with $74,246$ parameters controlled the ``Half-Cheetah'' planar robot from the PyBullet simulator \cite{coumans2016pybullet}; this simulation and NN architecture \change{are} modeled after the Half-Cheetah deep RL benchmark of \citet{finn2017model}. The NN policy received as input 26 variables consisting of the robot's position, velocity, and the angle and angular velocity of its joints; the policy controlled the robots through NN outputs that are interpreted as torques applied to each of the robot's 6 joints. In this domain, we characterized behavior as the final horizontal offset of the robot after a fixed number of simulation steps. All population sizes were $10,000$; other hyperparameters and algorithmic details can be found in supplemental section \ref{sec:experimental_details}.

Figure \ref{fig:cheetah_hist} shows the distribution of behaviors in the population over training time for standard ES as well as both evolvability ES variants.
First, these results show that, perhaps surprisingly, the results from the interference pattern task generalize to domains with complex dynamics and large NNs.
In particular, evolvability ES discovered policies such that small mutations resulted in diametrically opposed behaviors, i.e.\ approximately half of all mutations caused the robot to move left, while the others resulted in it moving right. 
\change{Videos of evaluations of these policies (available at \url{http://t.uber.com/evolvabilityes})} show the striking difference between the rightward and leftward behaviors, highlighting the non-triviality in ``fitting'' both of these policies into the same NN.
Interestingly, both evolvability ES variants produced similar distributions of behavior, even though MaxEnt-EES's objective is explicitly intended to incentivize more uniform distributions of behavior.

\begin{figure*}
\vspace{-0.17in}
\centering
\includegraphics[width=\linewidth]{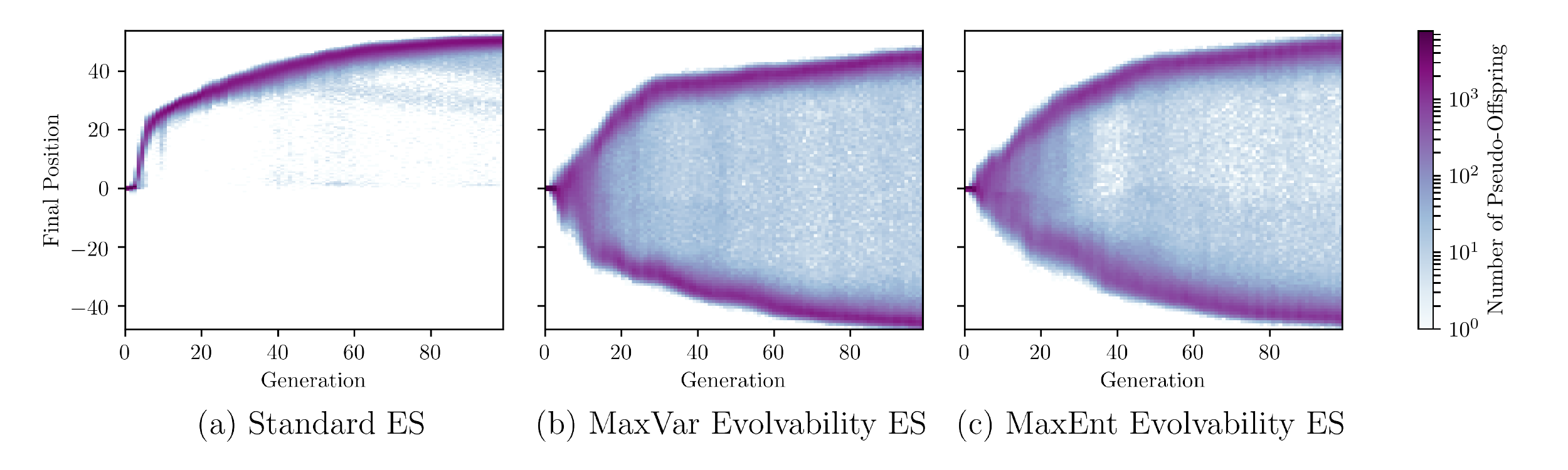}
\vspace{-0.35in}
\caption{\textbf{Distribution of behaviors across evolution in the 2-D locomotion domain.} Heat-maps of the final horizontal positions of 10,000 policies sampled from the population distribution are shown for each generation over training time for representative runs of each method. These plots suggest that both evolvability ES methods discovered policies that travel nearly as far both backwards and forwards as standard ES travels forward alone \change{(note that standard ES is rewarded for traveling forward)}.}
\label{fig:cheetah_hist}
\end{figure*}

Supplemental figure \ref{fig:cheetah_cmp} shows raw locomotion ability over training in this domain, i.e.\ the mean distance from the origin randomly sampled policies travel. Note that this is the metric optimized by standard ES (and does not reflect evolvability), and standard ES policies indeed moved further on average than both MaxVar- and MaxEnt-EES ($p < 0.05$, 12~runs; \emph{all tests unless otherwise specified are Mann-Whitney U Tests}).
However, on average, the MaxVar- and MaxEnt-EES variants yielded policies which moved $92.1\%$ and $93.9\%$ as far on average as those found by standard ES, so the performance price for evolvability is minimal.
There was no significant difference between MaxVar- and MaxEnt-EES ($p > 0.1$, 12~runs).

\vspace{-0.1cm}
\subsubsection{Comparison to MAML}

Evolvability ES takes inspiration from \change{MAML's ability to} successfully find policies near in parameter space to \change{diverse} behaviors \cite{finn2017model} and \change{to (like evolvability-inspired methods) adapt quickly to new circumstances}. Two questions naturally arise: (1) how does evolvability ES compare to MAML in enabling fast adaptation, and (2) is MAML itself drawn towards evolvable parts of the search space, i.e.\ are mutations of MAML solutions similar to those of evolvability ES solutions? 

To address these questions, MAML was trained in the 2-D locomotion task. Recall that MAML learns a policy which can readily adapt to a new task through further training. MAML, unlike evolvability ES, requires a distribution of training tasks \removedJoel{to adapt to} rather than a BC that recognizes distinct behaviors.
\change{MAML's task distribution consisted of two tasks of equal probability, i.e.\ to walk left or right.}
MAML additionally makes use of per-timestep rewards (rather than per-evaluation information as in EC), given here as the distance travelled that timestep in the desired direction. All such details were similar to \change{MAML's canonical application to this domain \cite{finn2017model}}; see supplemental section \ref{sec:experimental_details} for more information.

To explore how evolvability ES compares to MAML in enabling fast adaptation, we iteratively chose tasks, here directions to walk in, and allowed each algorithm to adapt to each given task.
For evolvability ES, when given a task to adapt to, the central individual was mutated 40 times, and the mutation performing best was selected (and its reported performance calculated from 10 separate evaluations with different random seeds).
For MAML, in an RL environment like the 2-D locomotion task, first the unmodified learned policy was evaluated in a new task (many times in parallel to reduce noise in the gradient), and then a policy-gradient update modified the MAML solution to solve the new task.
During such test-time adaptation, MaxEnt-EES performed statistically significantly better than MAML ($p < 0.005$, 12 runs), walking further on average in the desired direction after mutation and selection than MAML after adaptation. There was no significant difference between MaxVar-EES and MAML ($p > 0.1$, 12 runs), and MaxEnt-EES was significantly better than MaxVar-EES ($p < 0.05$, 12 runs). 
The conclusion is that evolvability ES in this domain is competitive with the leading gradient-based approach for meta-learning.

To compare what areas of the search space are discovered by MAML and evolvability ES, we explored what
distribution of behaviors resulted from applying random mutations to the solutions produced by MAML.
First, we subjected MAML policies to mutations with the same distribution as that used by evolvability ES, and then recorded the final positions of each of these offspring policies.
Figure \ref{fig:maml_hist} shows the resulting distributions of behaviors for representative runs of MAML and both EES-variants. Qualitatively, the MAML policies are clumped closer to the origin (i.e.\ they typically move less far), and have less overall spread (i.e.\ their variance is lower).
Both of these observations are statistically significant over 12 runs of each algorithm: The mean final distance from the origin for both evolvability ES variants was higher than for MAML ($p < 0.00005$ for both), and the variance of the distribution of final distances from the origin was higher for both evolvability ES variants than for MAML ($p < 0.00005$ for both). We also tested smaller perturbations of the MAML policy, under the hypothesis that gradient steps may be smaller than evolvability ES mutations, but found similar results; see supplemental figure \ref{fig:maml_hist_small}.
These two results together suggest that evolvability ES does indeed find more evolvable representations than MAML.

\begin{figure*}
\centering
\includegraphics[width=\linewidth]{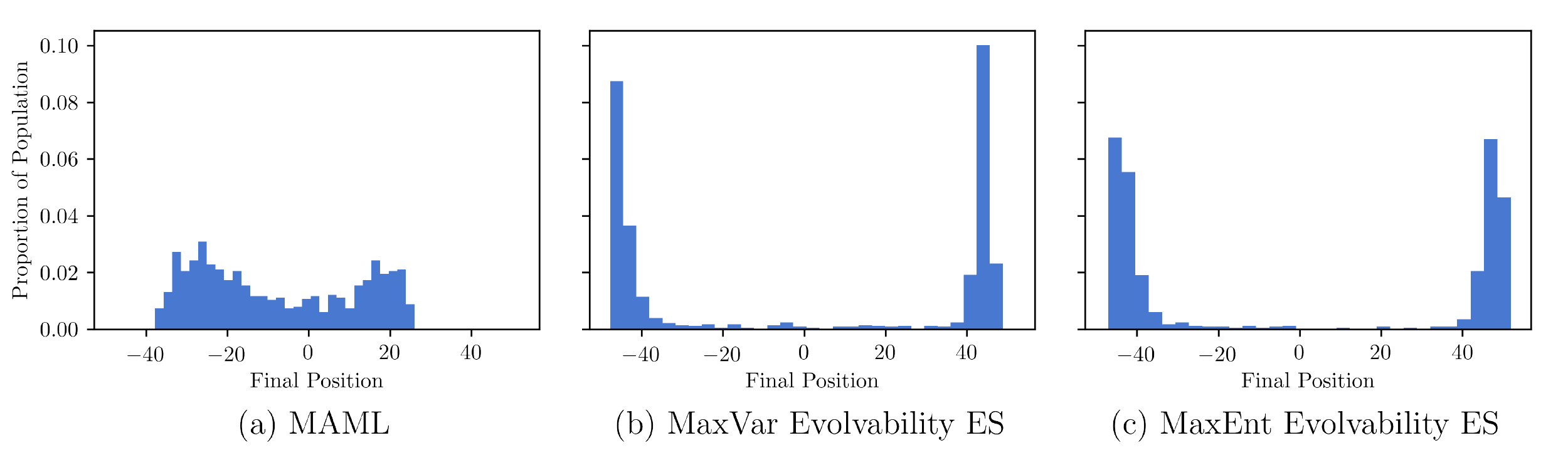}
\vspace{-0.35in}
\caption{\textbf{Distribution of behaviors compared to MAML in the 2-D locomotion domain.} Histograms of the final $x$ positions of 1,000 policies sampled from the final population distribution are shown each variant of evolvability ES, as well as for perturbations of MAML solutions of the same size as those of evolvability ES. These plots suggest that both evolvability ES methods discovered more evolvable regions of parameter space than did MAML.}
\label{fig:maml_hist}
\end{figure*}

\vspace{-0.2cm}
\subsection{3-D Locomotion Task}

The final challenge domain applied evolvability ES to a more complex 3-D locomotion task that controls a 
four-legged robot; here evolvability ES aimed to produce a genome with offspring that can locomote in any direction. In particular, in the 3-D locomotion task, a NN policy with $75,272$ parameters controlled the ``Ant'' robot from the PyBullet simulator \cite{coumans2016pybullet}, modeled after the domain in the MuJoCo simulator \cite{mujoco} that serves as a common benchmark in deep RL \cite{finn2017model}. The $28$ variables the policy received as inputs consisted of the robot's position and velocity, and the angle and angular velocity of its joints; the NN output torques for each of the ant's $8$ joints.
In this domain, a policy's BC was its final $(x,y)$ position.

Figure \ref{fig:ant_hist} shows the population of behaviors after training \change{(note that standard ES is rewarded for traveling forward)}. Interestingly,
both evolvability ES variants formed a ring of policies, i.e.\ they found parameter vectors from which nearly \emph{all directions} of travel
were reachable through mutations.
For evolvability ES there were few behaviors in the interior of the ring (i.e.\ few mutations were degenerate), whereas standard ES exhibited a trail of reduced performance (i.e.\ its policy was less robust to mutation). Supporting this observation, for standard ES more policies ended their evaluation less than $5$ units from the origin than did those sampled from either variant of evolvability ES ($p < 0.01$, 12~runs).

\begin{figure*}
\centering
\vspace{-0.17in}
\includegraphics[width=\linewidth]{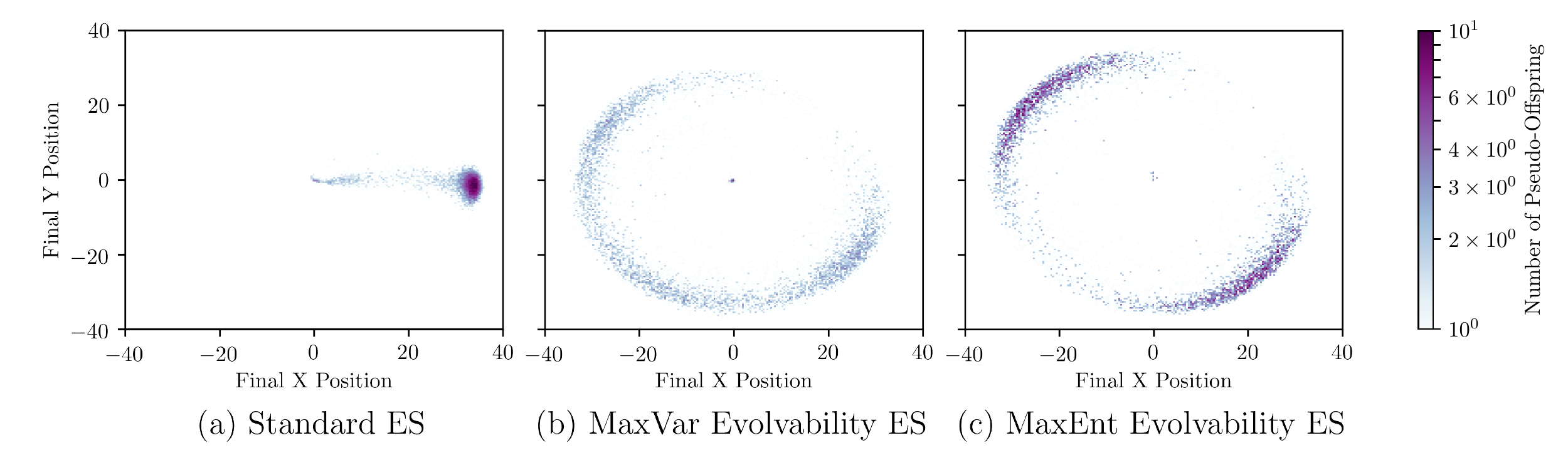}
\vspace{-0.35in}
\caption{\textbf{Distribution of behaviors in the final population in the 3-D locomotion domain.} Shown are heat-maps \change{(taken from representative runs)} of the final positions of 10,000 policies sampled from the population distribution at generation 100 for each of (a) standard ES, (b) MaxVar-EES, and (c) MaxEnt-EES. These plots suggest that both evolvability ES variants successfully found policies which moved in many different directions, and roughly as far as standard ES traveled in the positive $x$ direction alone.}
\label{fig:ant_hist}
\end{figure*}

Supplemental figure \ref{fig:ant_cmp} shows locomotion ability across training.
As in the 2-D locomotion domain, standard ES policies moved further on average than MaxVar- and MaxEnt-EES ($p < 0.05$, 12 runs).
However, the MaxVar- and MaxEnt-EES variants yielded policies which moved on average $85.4\%$ and $83.2\%$ as far as those from standard ES, which strikes a reasonable trade-off given the diversity of behaviors uncovered by evolvability ES relative to the single policy of standard ES.
There was no significant difference in performance between MaxVar- and MaxEnt-EES ($p > 0.1$, 12 runs).

\subsubsection{Seeding Further Evolution}

While previous experiments demonstrate that evolvability ES enables adaptation to new tasks
without further evolution, a central motivation for evolvability is to accelerate evolution in
general. 
As an initial investigation of evolvability ES's ability to seed further evolution, we used trained populations from evolvability ES as initializations for standard ES, and rewarded the agent for walking as far as possible in a particular direction, here, for simplicity, the positive $x$ direction.

Figure \ref{fig:adapt_ent_hist} shows a heat-map of behavior characteristics for the pre-trained populations produced by MaxEnt-EES, as well as \change{heat-maps} following up to twenty standard ES updates (for a similar MaxVar-EES plot see supplemental figure \ref{fig:adapt_var_hist}).
The population successfully converged in accordance with the new selection pressure, and importantly, it adapted much more quickly than randomly initialized population, as shown in figure \ref{fig:adapt_cmp}.
Indeed, further standard ES training of both variants of evolvability ES resulted in policies which moved statistically significantly further in the positive $x$ direction than a random initialization after $5$, $10$, and $20$ generations of adaptation ($p < 0.0005$, 12 runs). There was no significant difference between MaxVar- and MaxEnt-EES ($p > 0.1$, 12 runs).

\begin{figure}
\vspace{-0.2in}
  \centering
  \includegraphics[width=0.8\linewidth]{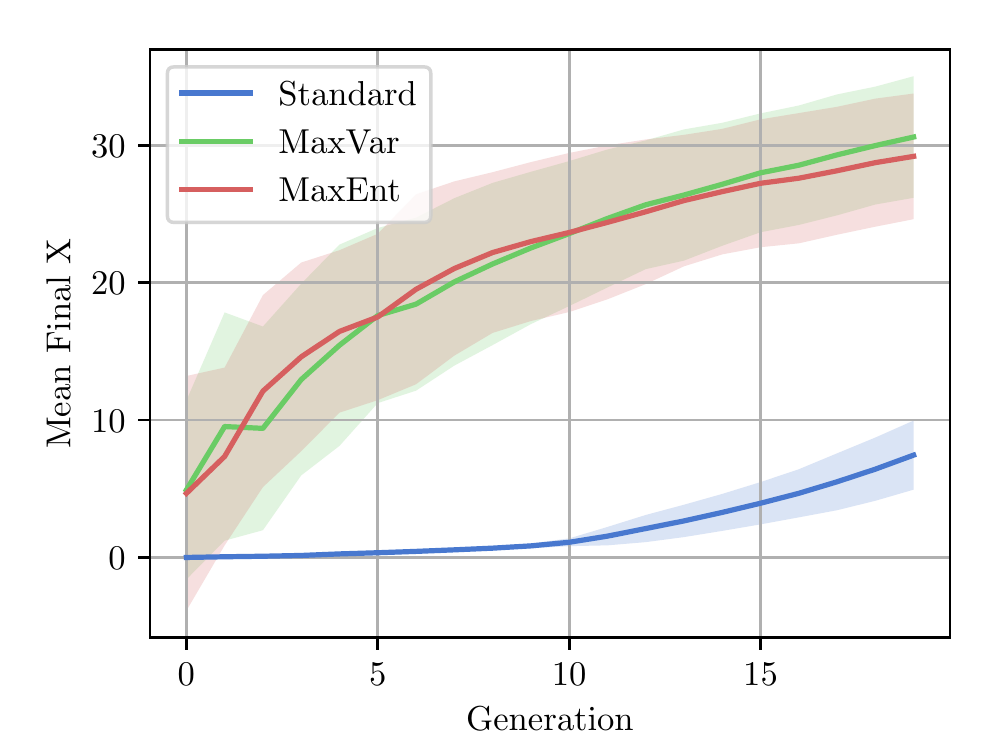}
  \vspace{-0.15in}
\caption{\textbf{Adaptation performance.} The plot compares the performance of standard ES adapting to a new behavior when seeded with a random initialization, or initializations from completed runs of MaxVar- and MaxEnt-EES. Mean and standard deviation over 12 runs shown. Populations trained by both methods of Evolvability ES evolved more quickly than randomly initialized populations.}
\label{fig:adapt_cmp}
\end{figure}

\begin{figure*}
\centering
\vspace{-0.17in}
\includegraphics[width=\linewidth]{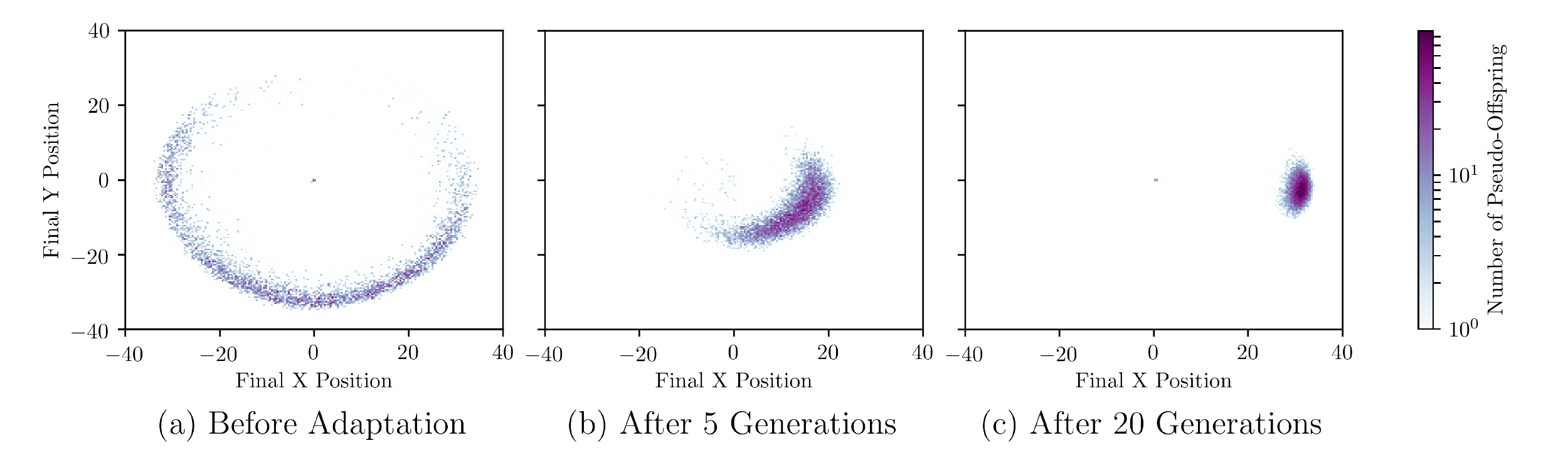}
\vspace{-0.35in}
\caption{\textbf{Distribution of behaviors during adaptation in the 3-D locomotion domain.} Heat-maps are shown of the final positions of 10,000 policies sampled from the population distribution initialized with MaxEnt-EES, and adapted to move in the positive $x$ direction with standard ES over several generations. These plots suggest that MaxEnt-EES successfully found policies that could quickly adapt to perform new tasks. See supplemental figure \ref{fig:adapt_var_hist} for the MaxVar version.}
\label{fig:adapt_ent_hist}
\end{figure*}

\subsection{Optimizing Population-Level Evolvability} \label{sec:mixtures}

\change{Finally, we explore (in the 3-D locomotion domain)} the potential to evolve what \citet{wilder2015reconciling} term \emph{population-level evolvability}, i.e.\ to consider evolvability as the sum of behaviors reachable from possible mutations of all individuals within a \emph{population}; this is contrast to 
the individual evolvability targeted by evolvability search and evolvability ES as described so far (both reward behavioral diversity accessible from a single central individual). E.g.\ in nature different organisms are primed to evolve in different directions, and one would not expect a dandelion to be able to quickly adapt into a dinosaur, though both may be individually highly-evolvable in certain phenotypic dimensions.



Indeed, for some applications, allowing for many separately-evolvable individuals might be more effective than optimizing for a single evolvable individual, e.g.\ some environments may entail too many divergent phenotypes for the neighborhood around a single central individual to encompass them all. Interestingly, the idea of joint optimization of many individuals for maximizing evolvability is naturally compatible with evolvability ES. The main insight is recognizing that the population distribution need not be unimodal; 
as an initial investigation of this idea, we experimented with a population distribution that is the mixture of two unimodal distributions, each representing a different species of individuals.




As a proof of concept, we compared two multi-modal variants of evolvability ES, where the different modes of a Gaussian Mixture Model (GMM; \cite{geoffrey:finite}) distribution can learn to specialize.
In the first variant (vanilla GMM), we equipped Evolvability ES with a multi-modal GMM population, where each mode was equally likely and was separately randomly initialized.
In the second variant (splitting GMM), we first trained a uni-modal population (as before) with evolvability ES until convergence, then seeded a GMM from it, by initializing the new component means from independent samples from the pre-trained uni-modal population (to break symmetry).
\change{The second variant models speciation, wherein a single species splits into subspecies that can further specialize.}
A priori, it is unclear whether the different modes will specialize into different behaviors, because the only optimization pressure is for the \emph{union} of both modes to cover a large region of behavior space. However, one might expect that more total behaviors be covered if the species \emph{do} specialize.

Figure \ref{fig:fork_ent_hist} and supplemental figure \ref{fig:fork_var_hist} show the qualitative performance of these two variants.
\change{These figures first show that specialization \emph{did} emerge in both variants, without direct pressure to do so, although
there was no significant difference in how the speciated version covered the space relative to individual-based evolability ES ($p > 0.1$, 6 runs). This result shows that intriguingly it is possible to directly optimize population-level evolvability, but its benefit may come only from domains with richer possibilities for specialization (e.g.\ an environment in which it is possible to either swim or fly), a direction to be explored in future work.}



\begin{figure*}
\centering
\vspace{-0.17in}
\includegraphics[width=\linewidth]{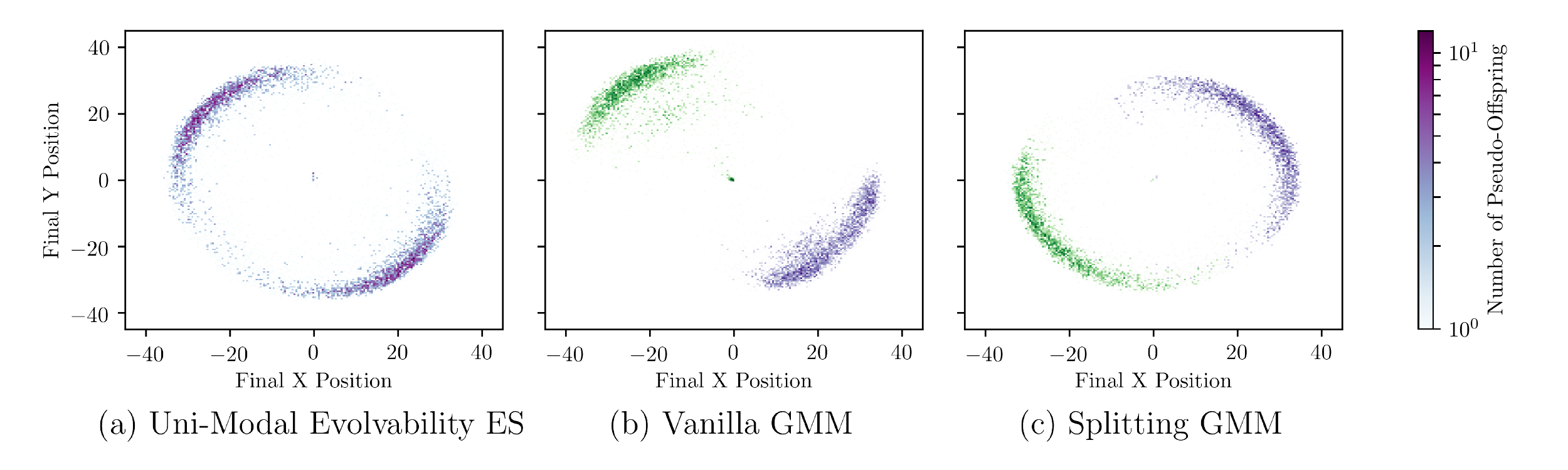}
\vspace{-0.35in}
\caption{\textbf{Distribution of behaviors for uni- and multi-modal MaxEnt-EES variants}. Heat-maps are shown of the final positions of 10,000 policies sampled from the population at generation 100 of MaxEnt-EES, with (a) one component and (b) a vanilla GMM with two components. Also shown is the result of (c) splitting the trained single component into two components, and evolving for 20 additional generations. These plots suggest \alexchange{that both bi-modal variants performed similarly}. See supplemental figure \ref{fig:fork_var_hist} for the MaxVar version.}
\label{fig:fork_ent_hist}
\end{figure*}


\vspace{-0.1in}

\section{Discussion and Conclusion}



This paper's results highlight that, surprisingly, it is possible to directly, efficiently, and successfully optimize evolvability in the space of large neural networks. 
One key insight enabling such efficient optimization of evolvability is that there exist deep connections between evolvability and parameter-space meta-learning algorithms like MAML, a gradient-based approach to meta-learning popular in traditional ML.
These results suggest that not only may MAML-like algorithms be used to uncover evolvable policies, evolutionary algorithms that in general incline towards evolvability may 
serve as a productive basis for meta-learning.
That is, evolutionary meta-learning has largely focused on
plastic neural networks \cite{stanley2003evolving,floreano1996evolution,soltoggio2008evolutionary}, but evolvable genomes that are primed to quickly adapt are themselves a form of meta-learning (as MAML and evolvability ES demonstrate); such evolvability could be considered instead of, or complementary to, NNs that adapt online.

The results of this paper also open up many future research directions. One natural follow-up concerns the use of evolvability as an auxiliary objective to complement novelty-driven ES \cite{conti2018neurips,cuccu2011novelty} or objective-driven ES. The intuition is that increased evolvability will catalyze the accumulation of novelty or progress towards a goal.
Interestingly, the same samples used for calculating novelty-driven or objective-driven ES updates can be \emph{reused} to estimate the gradient of evolvability; in other words, such an auxiliary objective could be calculated with no additional domain evaluations. Additionally, the population-level formulation of evolvability ES is itself similar to novelty-driven ES, and future work could compare them directly.

Interestingly, while part of Evolvability ES's origin comes from inspiration from gradient-based ML (i.e.\ MAML), it also offers the opportunity to inspire new gradient-based algorithms: A reformulation of the evolvability ES loss function could enable a policy gradients version of evolvability ES which exploits the differentiability of the policy to decrease the variance of gradient estimates (supplemental section \ref{sec:pg_ees} further discusses this possibility).
The conclusion is that Evolvability ES is a new and scalable addition to the set of tools for exploring, encouraging, and studying evolvability, one we hope will continue to foster cross-pollination between the EC and deep RL communities.

\clearpage
\bibliographystyle{plainnat}
\bibliography{main}

\clearpage
\beginsupplement

\section*{Supplemental Information}

Included in the supplemental information are experimental details and
hyperparameters for all algorithms (section \ref{sec:experimental_details}); a description of stochastic computation graphs and how we extended them (sections \ref{sec:computation_graphs} and \ref{sec:nested_computation_graphs}); 
and the particular stochastic computation graphs that enable calculating loss and gradients for ES and Evolvability ES (section \ref{sec:evo_graphs}).

\section{Experimental Details} \label{sec:experimental_details}

This section presents additional plots useful for better understanding the performance of
evolvability ES, as well as further details and hyperparameters for all algorithms.

\subsection{Additional Plots}

Figures \ref{fig:cheetah_cmp} and \ref{fig:ant_cmp} display the mean distance from the origin of both standard ES and both variants of evolvability ES, on the 2D and 3D locomotion tasks respectively.

Figure \ref{fig:adapt_var_hist} highlights how the distribution of behaviors changes during meta-learning test-time to quickly adapt to the task at hand for MaxVar-EES (see figure \ref{fig:adapt_ent_hist} for the MaxEnt version).

Figure \ref{fig:fork_var_hist} contrasts the two variants of multi-modal MaxVar-ES (see figure \ref{fig:fork_ent_hist} for the MaxEnt version).

Figure \ref{fig:maml_hist_small} compares perturbations of central evolvability ES policies to perturbations of MAML policies with a standard deviation 4 times smaller than that of EES.

\begin{figure}
  \centering
  \includegraphics[width=.9\linewidth]{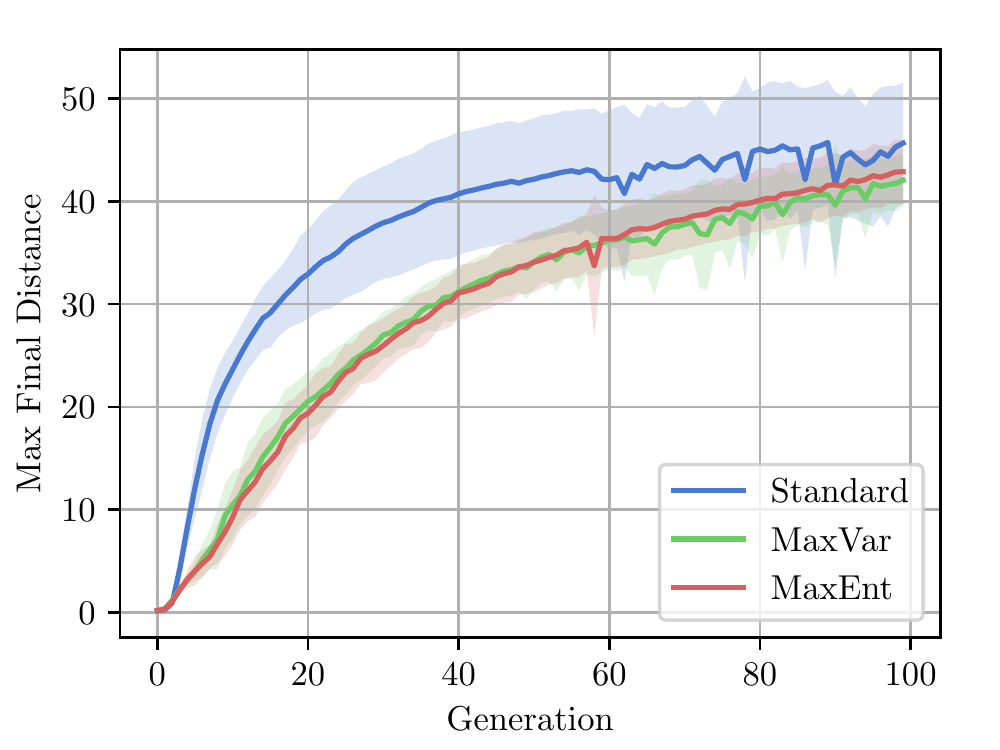}
  \vspace{-0.2in}
\caption{\textbf{2-D locomotion ability across training.} This plot compares raw locomotion ability of policies sampled from standard ES, MaxVar-EES, and MaxEnt-EES during training. Each curve plots the mean final distance from the origin over 10,000 samples from the population distribution. The error bars indicate standard deviation over the 12 training runs. While standard ES learned slightly faster, this plot shows that both evolvability ES variants found policies which moved almost as far as standard ES in this domain, despite encoding both forwards and backwards-moving policies.}
\label{fig:cheetah_cmp}
\end{figure}

\begin{figure}
  \centering
  \includegraphics[width=.9\linewidth]{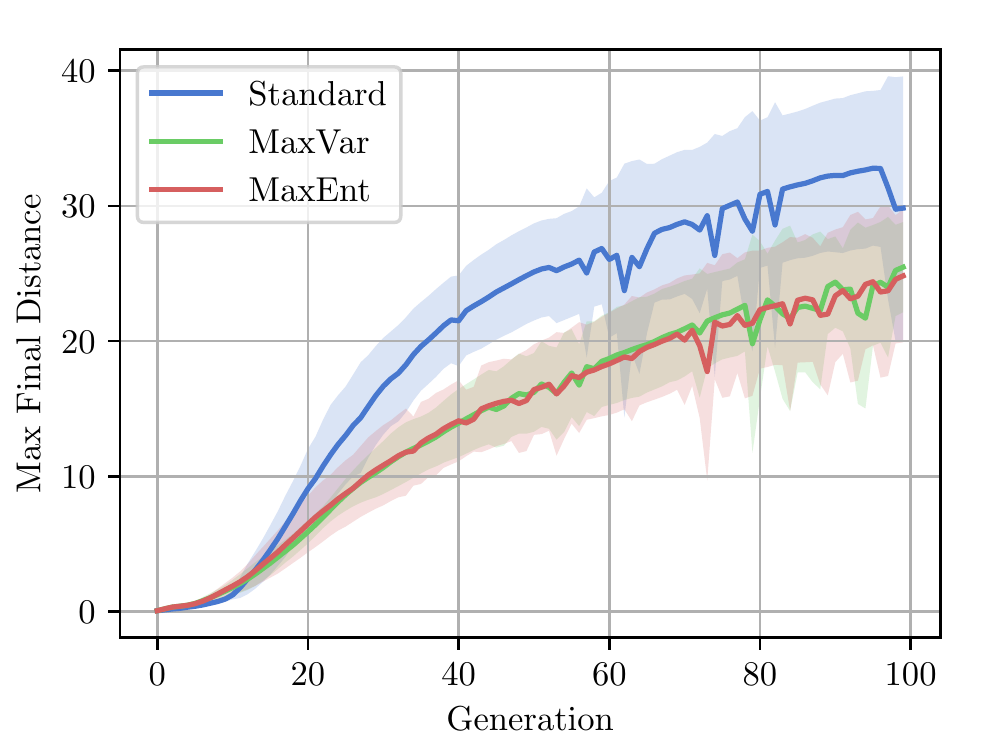}
  \vspace{-0.15in}
\caption{\textbf{3-D locomotion ability across training.} This plot compares raw locomotion ability of policies samples from standard ES, MaxVar-EES, and MaxEnt-EES. Each curve shows the mean final distance from the origin over 10,000 samples from the population distribution during training. Mean and standard deviation over 12 runs shown. While standard ES learned its single forward-moving policy more quickly, this plot highlights that both evolvability ES variants found policies which moved almost as far as standard ES on the ant domain, despite encoding policies that move in many more directions.}
\label{fig:ant_cmp}
\end{figure}

\begin{figure*}
\centering
\includegraphics[width=\linewidth]{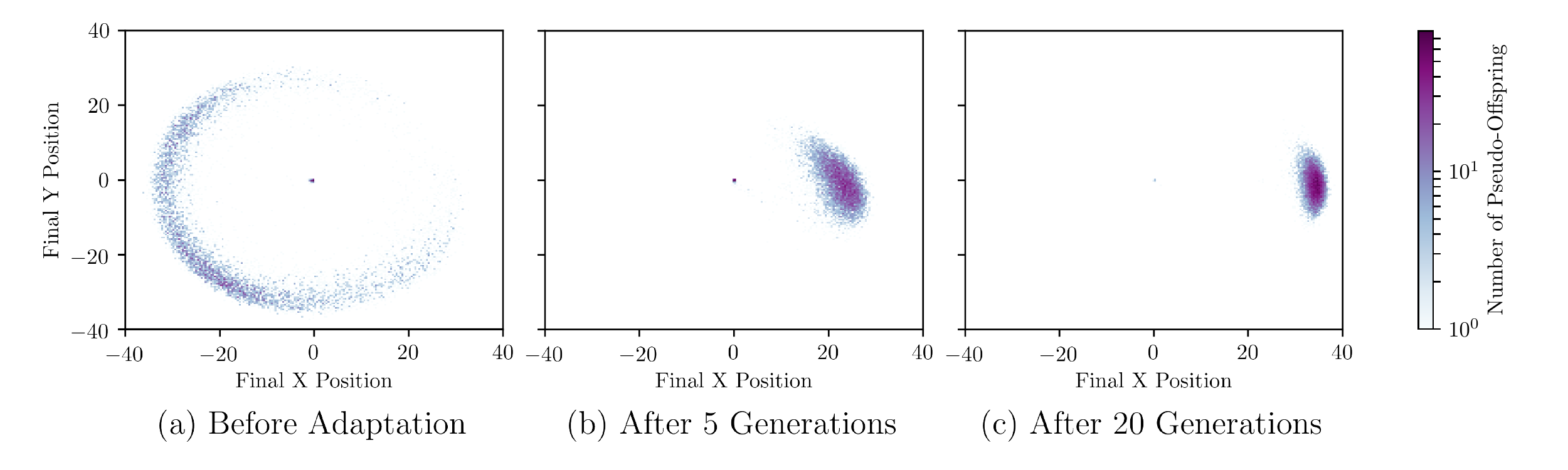}
\caption{\textbf{Distribution of behaviors during adaptation in the 3D locomotion domain.} Heat-maps of the final positions of 10,000 policies sampled from the population distribution initialized with MaxVar-EES, and adapted to move in the positive $x$ direction with Standard ES over several generations. These plots suggest that MaxVar-EES successfully found policies which could quickly adapt to perform new tasks. See figure \ref{fig:adapt_ent_hist} for the MaxEnt version.}
\label{fig:adapt_var_hist}
\end{figure*}

\begin{figure*}
\centering
\includegraphics[width=\linewidth]{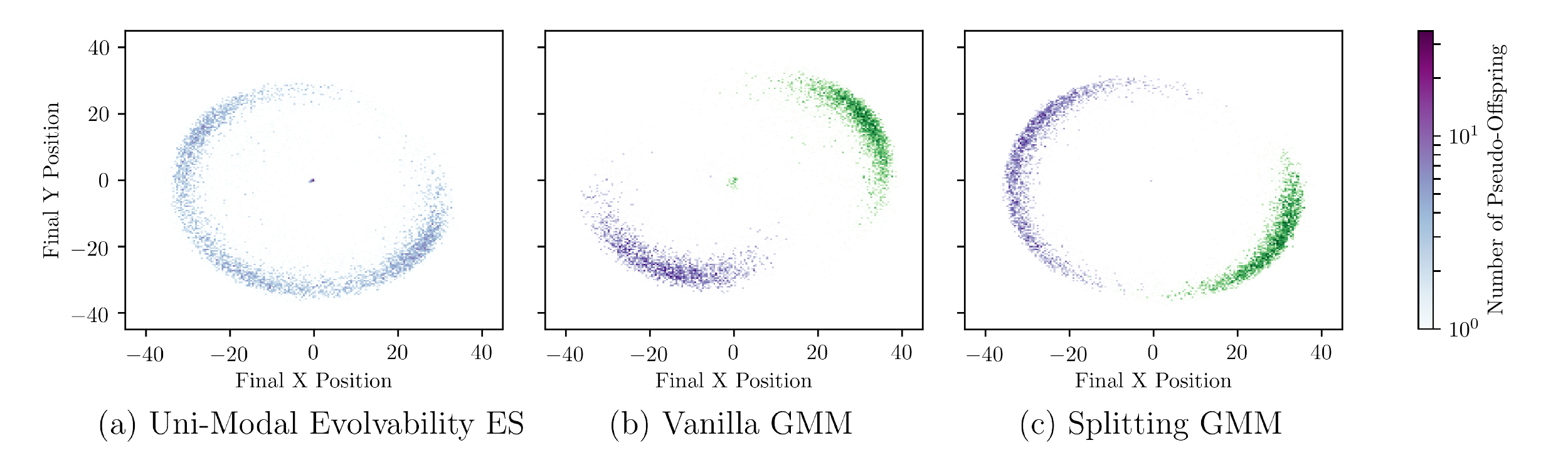}
\caption{\textbf{Distribution of behaviors for uni- and multi-modal MaxVar-EES variants}. Heat-maps of the final positions of 10,000 policies sampled from the population distribution at generation 100 of MaxVar-EES, with (a) one component and (b) a vanilla GMM with two components. Also shown is the result of splitting the single component in (a) into two components and evolving for 20 additional generations. These plots suggest that both bi-modal variants performed similarly. See figure \ref{fig:fork_ent_hist} for the MaxEnt version.}
\label{fig:fork_var_hist}
\end{figure*}

\begin{figure*}
\centering
\includegraphics[width=\linewidth]{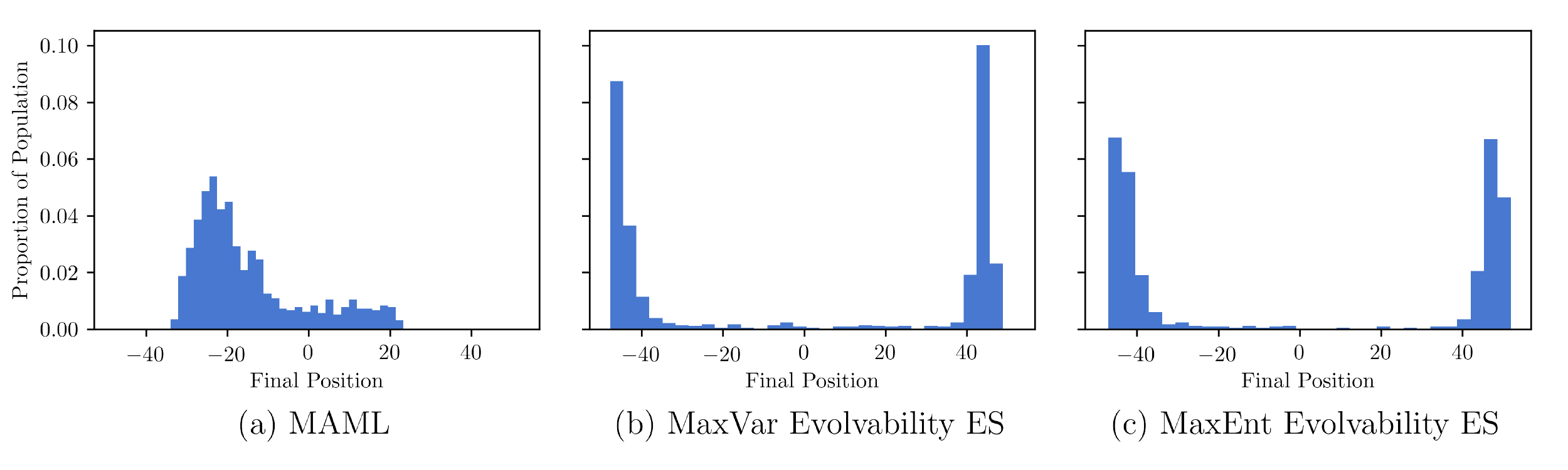}
\caption{\textbf{Distribution of behaviors compared to MAML in the 2-D locomotion domain. Histograms of the final $x$ positions of 1,000 policies sampled from the final population distribution are shown each variant of evolvability ES, as well as for perturbations of MAML policies 4 times smaller than those of evolvability ES. Smaller perturbations of MAML did not change the fundamental result (and interestingly eliminated the potential of generating far-right-walking policies).}}
\label{fig:maml_hist_small}
\end{figure*}

\subsection{Hyperparameters and Training Details}

For standard ES, fitness was rank-normalized to take values symmetrically between $-0.5$ and $0.5$ at each generation before computing gradient steps.
For both variants of evolvability ES, BCs were whitened to have mean zero and a standard deviation of one at each generation before computing losses and gradients.
This was done instead of rank normalization in order to preserve density information for variance and entropy estimation.

A Gaussian kernel with standard deviation $1.0$ was used for the MaxEnt-EES to estimate the density of behavior given samples from the population distribution.

\subsubsection{Interference Pattern Details}

The interference pattern was generated by the function
\begin{equation}
    f(x) = 5 \sin \frac{x}{5} \sin{20x}.
\end{equation}

Hyperparameters for the interference pattern task are shown in Tables \ref{tab:maxvar_interference} and \ref{tab:maxent_interference}.

\begin{table}[h]
 \centering
\begin{tabular}{l r r r}
\toprule
\small Hyperparameter & \small Setting \\
\midrule
\small Learning Rate & \small 0.03 \\
\small Population Standard Deviation & \small 0.5 \\
\small Population Size & \small 500 \\
\bottomrule
\end{tabular}
 \caption{\textbf{MaxVar Hyperparameters: Interference Pattern Task.}}
\label{tab:maxvar_interference}
\end{table}

\begin{table}[h]
 \centering
\begin{tabular}{l r r r}
\toprule
\small Hyperparameter & \small Setting \\
\midrule
\small Learning Rate & \small 0.1 \\
\small Population Standard Deviation & \small 0.5 \\
\small Population Size & \small 500 \\
\small Kernel Standard Deviation & \small 1.0 \\
\bottomrule
\end{tabular}
 \caption{\textbf{MaxEnt Hyperparameters: Interference Pattern Task.}}
\label{tab:maxent_interference}
\end{table}

\subsubsection{Locomotion Task Details}

For the locomotion tasks, all environments were run deterministically and no action noise was used during training.
The only source of randomness was from sampling from the population distribution. Policies were executed in the environment for $1,000$ timesteps. For comparison to evolvability ES, the fitness function for standard ES was set to be the final $x$ position of a policy, rewarding standard ES for walking as far as possible in positive $x$ direction.

NNs for both 2D and 3D locomotion 
were composed of two hidden layers with $256$ hidden units each, resulting in $166.7$K total parameters,
and were regularized with L2 penalties during training.
Inputs to the networks were normalized to have mean zero and a standard deviation of one based the mean and standard deviation of the states seen in a random subset of all training rollouts, with each rollout having probability $0.01$ of being sampled.

Experiments were performed on a cluster system and were distributed across a pool of $550$ CPU cores shared between two runs of the same algorithm.
Each run took approximately $5$ hours to complete.

Hyperparameters for both variants of Evolvability ES are shown Tables \ref{tab:maxvar_locomotion} and \ref{tab:maxent_locomotion}.

\begin{table}[h]
 \centering
\begin{tabular}{l r r r}
\toprule
\small Hyperparameter & \small Setting \\
\midrule
\small Learning Rate & \small 0.01 \\
\small Population Standard Deviation & \small 0.02 \\
\small Population Size & \small 10,000 \\
\small L2 Regularization Coefficient & \small 0.05 \\
\bottomrule
\end{tabular}
 \caption{\textbf{MaxVar Hyperparameters: Locomotion Tasks.}}
\label{tab:maxvar_locomotion}
\end{table}

\begin{table}[h]
 \centering
\begin{tabular}{l r r r}
\toprule
\small Hyperparameter & \small Setting \\
\midrule
\small Learning Rate & \small 0.01 \\
\small Population Standard Deviation & \small 0.02 \\
\small Population Size & \small 10,000 \\
\small L2 Regularization Coefficient & \small 0.05 \\
\small Kernel Bandwidth & \small 1.0 \\
\bottomrule
\end{tabular}
 \caption{\textbf{MaxEnt Hyperparameters: Locomotion Tasks.}}
\label{tab:maxent_locomotion}
\end{table}

\subsubsection{MAML Details}
The MAML algorithm was run on the 2D locomotion task using the same fully-connected neural network with 2 hidden layers of 256 hidden units each as that used in evolvability ES.
There are two main differences between our usage of MAML and those typically done in the past (e.g.\ in \citet{finn2017model}).
First, we used the PyBullet simulator \cite{coumans2016pybullet} as opposed to the more canonical MuJoCo simulator \cite{mujoco}.
Second, typically locomotion tasks have associated with them an \emph{energy penalty}, supplementing the standard distance-based reward function.
Because it is unclear how to incorporate an energy penalty into a (potentially vector-valued) behavior characteristic, we used no energy penalty in our evolvability ES experiments. Consequently, for a fairer comparison we also used no energy penalty in our MAML experiments.
Table \ref{tab:maml_details} displays additional hyperparameters for MAML.

\begin{table}[h]
 \centering
\begin{tabular}{l r r r}
\toprule
\small Hyperparameter & \small Setting \\
\midrule
\small Adaptation Learning Rate & \small 0.1 \\
\small Conjugate Gradient Damping & \small 1e-5 \\
\small Conjugate Gradient Iterations & \small 10 \\
\small Max Number of Line Search Steps & \small 15 \\
\small Line Search Backtrack Ratio & \small 0.8 \\
\small Discount Factor & \small 0.99 \\
\small Adaptation Batch Size & \small 10 \\
\small Outer Batch Size & \small 40 \\
\bottomrule
\end{tabular}
 \caption{\textbf{MAML Hyperparameters.}}
\label{tab:maml_details}
\end{table}

\section{Stochastic Computation Graphs} \label{sec:computation_graphs}

A stochastic computation graph, as defined in \citet{schulman_2015}, is a directed acyclic graph consisting of fixed input nodes, deterministic nodes representing functions of their inputs, and stochastic nodes representing random variables distributed conditionally on their inputs.

A stochastic computation graph $\mathcal{G}$ represents the expectation (over its stochastic nodes $\{z_i\}$) of the sum of its output nodes $\{f_i\}$, as a function of its input nodes $\{x_i\}$:
\begin{equation}
    \mathcal{G}(x_1, \ldots, x_l) = \mathbb{E}_{z_1, \ldots, z_m}\left [ \sum_{i=1}^n f_i \right ]
\end{equation}

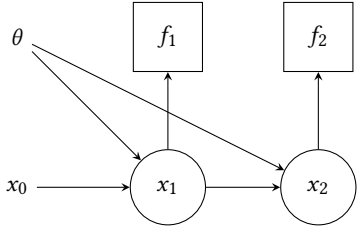
\begin{figure}
    \centering
    \begin{tikzpicture}[square/.style={regular polygon,regular polygon sides=4}]
        \node (x0) at (0, 0) {$x_0$};
        \node (x1) at (2, 0) [minimum size=1cm,draw,circle] {$x_1$};
        \node (x2) at (4, 0) [minimum size=1cm,draw,circle] {$x_2$};
        \node (theta) at (0, 2) {$\theta$};
        \node (f1) at (2,2) [minimum size=1.3cm,draw,square] {$f_1$};
        \node (f2) at (4,2) [minimum size=1.3cm,draw,square] {$f_2$};
        \draw [->,>=stealth] (x0) edge (x1) (x1) edge (x2) (x1) edge (f1) (x2) edge (f2) (theta) edge (x1) (theta) edge (x2);
    \end{tikzpicture} 
    \caption{Example stochastic computation graph: input nodes are depicted with no border, deterministic nodes with square borders, and stochastic nodes with circular borders.}
    \label{fig:mrp_original}
\end{figure}

For example, consider the stochastic computation graph in figure \ref{fig:mrp_original}, reproduced from \cite{schulman_2015}. This graph $\mathcal{G}$ represents the expectation
\begin{equation}
    \mathcal{G}(x_0, \theta) = \mathbb{E}_{x_1, x_2}\left [ f_1(x_1) + f_2(x_2) \right ],
\end{equation}
where $x_1 \sim p(\cdot;x_0, \theta)$ and $x_2 \sim p(\cdot;x_1, \theta)$.

A key property of stochastic computation graphs is that they may be differentiated with respect to their inputs.
Using the score function estimator \cite{fu_2006}, we have that
\begin{equation}
    \begin{split}
    \nabla_\theta \mathcal{G}(x_0, \theta) =  \mathbb{E}_{x_1, x_2}\biggl [\nabla_\theta \log p(x_1;\theta, x_0)(f_1(x_1) + f_2(x_2)) \\
    + \nabla_\theta \log p(x_2;\theta, x_1) f_2(x_2) \biggr ].
    \end{split}
\end{equation}
\citet{schulman_2015} also derive \emph{surrogate} loss functions for stochastic computation graphs, allowing for implementations of stochastic computation graphs with existing automatic differentiation software.

For example, given a sample $\{x_1^i\}_{1 \le i \le N}$ of $x_1$ and $\{x_2^i\}_{1 \le i \le N}$ of $x_2$, we can write
\begin{equation}
    \hat{L}(\theta) = \frac{1}{N} \sum_i \log p(x_1^i;\theta, x_0)(f_1(x_1^i) + f_2(x_2^i)) + \log p(x_2^i;\theta, x_1^i) f_2(x_2^i).
\end{equation}
Now to estimate $\nabla_\theta \mathcal{G}(x_0, \theta)$, we have
\begin{equation}
    \nabla_\theta \mathcal{G}(x_0, \theta) \approx \nabla_\theta \hat{L}(\theta),
\end{equation}
which may be computed with popular automatic differentiation software.

\section{Nested Stochastic Computation Graphs} \label{sec:nested_computation_graphs}

We make two changes to the stochastic computation graph formalism:

\begin{enumerate}
    \item We add a third type of node which represents the expectation over one of its parent stochastic nodes of one of its inputs.
    We require that a stochastic node representing a random variable $z$ be a dependency of exactly one expectation node over $z$, and that every expectation node over a random variable $z$ depend on a stochastic node representing $z$.
    \item Consider a stochastic node representing a random variable $z$ conditionally dependent on a node $y$.
    Rather than expressing this as a dependency of $z$ on $y$, we represent this as a dependency between the expectation node over $z$ on $y$.
    Formally, this means all stochastic nodes are required to be leaves of the computation graph.
\end{enumerate}
Because ``nested stochastic computation graphs,'' as we term them, contain their expectations explicitly, they simply represent the sum of their output nodes (instead of the expected sum of their output nodes, as with regular stochastic computation graphs).

\begin{figure}
    \centering
    \begin{tikzpicture}[square/.style={regular polygon,regular polygon sides=4}]
        \node (x0) at (0, 0) {$x_0$};
        \node (x1) at (4, 0) [minimum size=1cm,draw,circle] {$x_1$};
        \node (x2) at (6, 0) [minimum size=1cm,draw,circle] {$x_2$};
        \node (theta) at (0, 2) {$\theta$};
        \node (f1) at (4,2) [minimum size=1.3cm,draw,square] {$f_1$};
        \node (f2) at (6,2) [minimum size=1.3cm,draw,square] {$f_2$};
        \node (E1) at (2,4) [minimum size=1cm,draw,double,ellipse] {$\mathbb{E}_{x_1}[+]$};
        \node (S1) at (4,4) [minimum size=1.3cm,draw,square] {$+$};
        \node (E2) at (6,4) [minimum size=1cm,draw,double,ellipse] {$\mathbb{E}_{x_2}[f_2]$};
        \draw [->,>=stealth] (x0) edge (E1) (x1) edge (E2) (x1) edge (f1) (x2) edge (f2) (theta) edge (E1) (theta) edge (E2) (f1) edge (S1) (f2) edge (E2) (E2) edge (S1) (S1) edge (E1);
    \end{tikzpicture} 
    \caption{Example nested stochastic computation graph: input nodes are depicted with no border, deterministic nodes with square borders,  stochastic nodes with circular borders, and expectation nodes with double elliptical borders.}
    \label{fig:mrp_nested}
\end{figure}
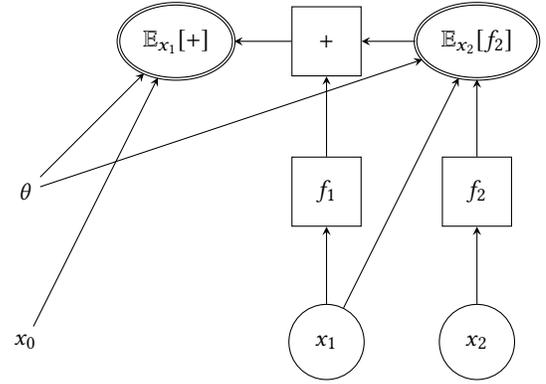

As an example, consider the nested stochastic computation graph $\mathcal{G}$ depicted in figure \ref{fig:mrp_nested}.
First, note that $\mathcal{G}$ is indeed a nested stochastic computation graph, because the stochastic nodes and expectation nodes correspond, and because all stochastic nodes are leaves of the graph.
Next, note that $\mathcal{G}$ is equivalent to the stochastic computation graph of figure \ref{fig:mrp_original} in the sense that it computes the same function of its inputs:
\begin{align} \label{eq:mrp_nested}
    \mathcal{G}(x_0, \theta) &= \mathbb{E}_{x_1}\left [ f_1(x_1) + \mathbb{E}_{x_2} \left [ f_2(x_2) \right ] \right ]\\
    &= \mathbb{E}_{x_1, x_2}\left [ f_1(x_1) + f_2(x_2) \right ]
\end{align}
The original stochastic computation graph formalism has the advantage of more clearly depicting conditional relationships, but this new formalism has two advantages:

\begin{enumerate}
    \item Nested stochastic computation graphs can represent arbitrarily nested expectations. We have already seen this in part with the example of figure \ref{fig:mrp_nested}, but we shall see this more clearly in a few sections.
    \item It is trivial to define surrogate loss functions for nested stochastic computation graphs. Moreover, these surrogate loss functions have the property that in the forward pass, they estimate the true loss function.
\end{enumerate}

Consider a nested stochastic computation graph $\mathcal{G}$ with input nodes $\{\theta\} \cup \{x_i\}$, and suppose we wish to compute the gradient $\nabla_\theta \mathcal{G}(\theta, x_1, \ldots, x_n)$. 
We would like to be able to compute the gradient of any node with respect to any of its inputs, as this would allow us to use the well-known backpropagation algorithm to compute $\nabla_\theta \mathcal{G}$.
Unfortunately, it is often impossible to write the gradient of an expectation in closed form; we shall instead estimate $\nabla_\theta \mathcal{G}$ given a sample from the stochastic nodes of $\mathcal{G}$.

\begin{figure}
    \centering
    \begin{tikzpicture}[square/.style={regular polygon,regular polygon sides=4}]
        \node (z) at (3, 0) [minimum size=1cm,draw,circle] {$z$};
        \node (y1) at (1.5, 1.5) [minimum size=1.3cm,draw,square] {$y_1$};
        \node (ydots) at (3, 1.5) {$\ldots$};
        \node (yl) at (4.5, 1.5) [minimum size=1.3cm,draw,square] {$y_l$};
        \node (xi1) at (0, 5) [minimum size=1.3cm,draw,square] {$\xi_1$};
        \node (xidots) at (0, 4.1) {$\vdots$};
        \node (xim) at (0, 3) [minimum size=1.3cm,draw,square] {$\xi_m$};
        \node (f) at (3, 3) [minimum size=1.3cm,draw,square] {$f$};
        \node (E) at (3, 4.5) [minimum size=1cm,draw,double,ellipse] {$\mathbb{E}_z [f]$};
        \draw [->,>=stealth] (y1) edge (f) (yl) edge (f) (xi1) edge (E) (xim) edge (E) (z) edge (ydots) (ydots) edge (f) (xidots) edge (E) (f) edge (E);
    \end{tikzpicture} 
    \caption{Nested stochastic computation graph with a single expectation node.}
    \label{fig:surrogate_loss}
\end{figure}
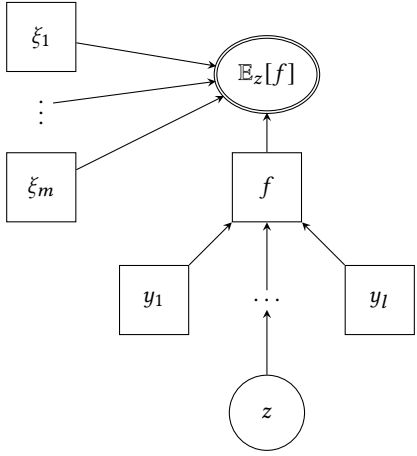

Suppose $\mathcal{G}$ has a output node $\mathbb{E}_z [f]$, the only expectation node in $\mathcal{G}$.
Suppose moreover that $\mathbb{E}_z [f]$ has inputs $\{\xi_i\}$ (apart from $f$) so $z \sim p(\cdot, \xi_1 \ldots \xi_m)$, and suppose $f$ has inputs $\{y_i\}$. Note that to satisfy the definition of a nested stochastic computation graph $f$ must ultimately depend on $z$, so we write $f$ as $f(y_1, \ldots, y_l; z)$. See figure \ref{fig:surrogate_loss} for a visual representation of $\mathcal{G}$.

If we wish to compute $\nabla_\omega \mathbb{E}_z \left [f(y_1, \ldots, y_l; z) \right ]$, using the likelihood ratio interpretation of the score function \cite{fu_2006} and given a sample $\{z_i\}_{1 \le i \le N}$ of $z$, we can write

\begin{equation} \label{eq:surrogate_loss}
    \hat{L}(\omega) = \frac{1}{N} \sum_i f(y_1, \ldots, y_l; z) \mathcal{L}(z_i),
\end{equation}
where $\mathcal{L}$ is the likelihood ratio given by

\begin{equation}
    \mathcal{L}(z_i) = \frac{p(z_i; \xi_1, \ldots, \xi_m)}{p(z_i; \xi'_1, \ldots, \xi'_m)},
\end{equation}
and setting $\xi'_i = \xi_i$ gives

\begin{align}
    \nabla_\omega \mathcal{L}(z_i) &= \frac{\nabla_\omega p(z_i; \xi_1, \ldots, \xi_m)}{p(z_i; \xi_1, \ldots, \xi_m)} \\
    &= \nabla_\omega \log p(z_i; \xi_1, \ldots, \xi_m).
\end{align}

Note that $f$ can either depend on $\omega$ directly, if $\omega \in \{y_i\}$, or through the distribution of $z$, if $\omega \in \{\xi_i\}$.
Differentiating, we have
\begin{equation}
    \nabla_\omega \hat{L}(\omega) = \frac{1}{N} \sum_i f(y_1, \ldots, y_l; z) \nabla_\omega \mathcal{L}(z_i)
    + \nabla_\omega f(y_1, \ldots, y_l; z) \mathcal{L}(z_i),
\end{equation}
and setting $\xi'_i = \xi_i$ we see that $\nabla_\omega \hat{L}(\omega)$ is an estimate of
\begin{equation}
    \nabla_\omega \mathbb{E}_z \left [f(y_1, \ldots, y_l; z) \right ]
\end{equation}

Generalizing this trick to an arbitrary nested stochastic computation graph $\mathcal{G}$, we see that creating a surrogate loss function $\hat{L}$ is as simple as replacing each expectation node with a sample mean as in Equation \ref{eq:surrogate_loss}, weighted by the likelihood ratio $\mathcal{L}(z_i)$.
Note that since $\mathcal{L}(z_i) = 1$, the surrogate loss is simply the method of moments estimate of the true loss. 

Considering again the graph of figure \ref{fig:mrp_nested}, we can construct a surrogate loss function
\begin{equation} \label{eq:mrp_nested_surrogate}
    \hat{L}(\theta, x_0) = \frac{1}{N} \sum_i \biggl (f_1(x_1^i) + \sum_j f_2(x_2^i) \mathcal{L}(x_2^i; \theta) \biggr ) \mathcal{L}(x_1^i; \theta, x_0).
\end{equation}
While this may not seem like very much of an improvement at first, it is insightful to note how similar the forms of Equations \ref{eq:mrp_nested} and \ref{eq:mrp_nested_surrogate} are. In particular, this similarity makes it straightforward to write a custom ``expectation'' operation for use in automatic differentiation software which computes the sample mean weighted by the likelihood ratio.

\section{Stochastic Computation Graphs for Standard ES and Evolvability ES} \label{sec:evo_graphs}

As mentioned in the main text of the paper, we can estimate the gradients of the standard ES and evolvability ES loss functions with the score function estimator because we can represent these loss functions as (nested) stochastic computation graphs. Figure \ref{fig:nes} shows a (nested) stochastic computation graph representing the standard ES loss function. This yields the following surrogate loss function for ES:
\begin{equation} \label{eq:es_surr}
    \hat{L}(\theta) = \frac{1}{N} \sum_i f(z_i) \mathcal{L}(z_i),
\end{equation}
where $\mathcal{L}(z_i)$ is the likelihood function.

Figure \ref{fig:max_variance_graph} shows a nested stochastic computation graph representing MaxVar-EES, yielding the following surrogate loss function:
\begin{equation} \label{eq:var_surr}
    \hat{L}(\theta) = \frac{1}{N} \sum_i \bm{B}(z_i)^2 \mathcal{L}(z_i; \theta)
\end{equation}
Finally, figure \ref{fig:max_entropy_graph} shows a nested stochastic computation graph representing the loss function for MaxEnt-EES, yielding the following surrogate loss function:
\begin{equation} \label{eq:ent_surr}
    \hat{L}(\theta) = -\frac{1}{N} \sum_i \log \left (\sum_j \varphi(\bm{B}(z')-z) \mathcal{L}(z_j; \theta) \right ) \mathcal{L}(z_i; \theta)
\end{equation}

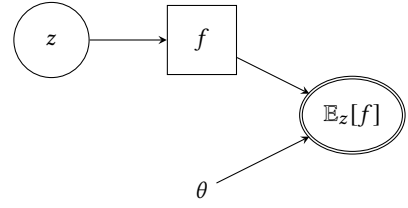
\begin{figure}
    \centering
    \begin{tikzpicture}[square/.style={regular polygon,regular polygon sides=4}]
        \node (z) at (0, 2) [minimum size=1cm,draw,circle] {$z$};
        \node (f) at (2, 2) [minimum size=1.3cm,draw,square] {$f$};
        \node (theta) at (2, 0) {$\theta$};
        \node (E) at (4, 1) [minimum size=1cm,draw,double,ellipse] {$\mathbb{E}_z [f]$};
        \draw [->,>=stealth] (z) edge (f) (f) edge (E) (theta) edge (E);
    \end{tikzpicture} 
    \caption{Nested stochastic computation graph representing Natural Evolution Strategies.}
    \label{fig:nes}
\end{figure}

\begin{figure}
    \centering
    \begin{tikzpicture}[square/.style={regular polygon,regular polygon sides=4}]
        \node (z) at (0, 2) [minimum size=1cm,draw,circle] {$z$};
        \node (B) at (2, 2) [minimum size=1.3cm,draw,square] {$\bm{B}$};
        \node (Sq) at (4, 2) [minimum size=1.3cm,draw,square] {$(\cdot)^2$};
        \node (theta) at (4, 0) {$\theta$};
        \node (E) at (6, 1) [minimum size=1cm,draw,double,ellipse] {$\mathbb{E}_z [(\cdot)^2]$};
        \draw [->,>=stealth] (z) edge (B) (B) edge (Sq) (Sq) edge (E) (theta) edge (E);
    \end{tikzpicture} 
    \caption{Nested stochastic computation graph representing the loss function of MaxVar-EES.}
    \label{fig:max_variance_graph}
\end{figure}

\begin{figure}
    \centering
    \begin{tikzpicture}[square/.style={regular polygon,regular polygon sides=4}]
        \node (z1) at (0, 3) [minimum size=1cm,draw,circle] {$z$};
        \node (B1) at (1.5, 3) [minimum size=1.3cm,draw,square] {$\bm{B}$};
        \node (z2) at (0, 0) [minimum size=1cm,draw,circle] {$z'$};
        \node (B2) at (1.5, 0) [minimum size=1.3cm,draw,square] {$\bm{B}$};
        \node (phi) at (1.5, 1.5) [minimum size=1.3cm,draw,square] {$\varphi$};
        \node (E2) at (3.2, 1.5) [minimum size=1cm,draw,double,ellipse] {$\mathbb{E}_z' [\varphi]$};
        \node (log) at (4.9, 1.5) {$-\log$};
        \node (fakelog) at (4.9, 1.5) [minimum size=1.3cm,draw,square] {};
        \node (E1) at (7, 1.5) [minimum size=1cm,draw,double,ellipse] {$\mathbb{E}_z [-\log]$};
        \node (theta) at (4.9, 0) {$\theta$};
        \draw [->,>=stealth] (z1) edge (B1) (z2) edge (B2) (B1) edge (phi) (B2) edge (phi) (phi) edge (E2) (E2) edge (fakelog) (fakelog) edge (E1) (theta) edge (E1) (theta) edge (E2);
    \end{tikzpicture} 
    \caption{Nested stochastic computation graph representing the loss function of MaxEnt-EES.}
    \label{fig:max_entropy_graph}
\end{figure}

\section{Policy Gradients Evolvability ES}\label{sec:pg_ees}
With MaxVar-EES, the loss function we wish to optimize is given by:
\begin{equation}
    J(\mu) = \mathbb{E}_{\theta \sim N(\mu, \sigma^2)}\left(\sum_{t=1}^\tau \mathbb{E}_{a_t \sim \pi(\cdot | s_t, \theta)} (B_\tau - \mathbb{E}_{\theta}(B_\tau))^2 \mathds{1}[t = \tau]\right)
\end{equation}
where $\tau$ is the length of an episode; the action $a_t$ at time $t$ is sampled from a distribution $\pi(\cdot|s_t, \theta)$ depending on the current state $s_t$ and the parameters of the current policy $\theta$; and $B_\tau$ is the behavior at time $\tau$.
Evolvability ES uses the following representation of the gradient of this function:
\begin{equation}
    \begin{split}
    \nabla_\mu J(\mu) = \mathbb{E}_{\theta \sim N(\mu, \sigma^2)}\biggl(\sum_{t=1}^\tau \mathbb{E}_{a_t \sim \pi(\cdot | s_t, \theta)} (B_\tau - \mathbb{E}_{\theta}(B_\tau))^2 \\
    \nabla_\mu \log N(\theta ; \mu, \sigma^2)\mathds{1}[t = \tau] \biggr)
    \end{split}
\end{equation}
Another way to write this same gradient is as follows, using the reparametrization trick:
\begin{equation}
    \begin{split}
    \nabla_\mu J(\mu) = \mathbb{E}_{\epsilon \sim N(0, \sigma^2)}\biggl(\sum_{t=1}^\tau \mathbb{E}_{a_t \sim \pi(\cdot | s_t, \mu + \epsilon)} (B_\tau - \mathbb{E}_{\theta}(B_\tau))^2 x\\
    \nabla_\mu \log \pi(a_t | s_t, \mu + \epsilon)\mathds{1}[t = \tau] \biggr)
    \end{split}
\end{equation}
This representation is reminiscent of the NoisyNet formulation of the policy gradient algorithm \cite{fortunato2017_noisynets}.
Evolvability ES estimates its representation of the gradient from samples like this:
\begin{equation}
\nabla_\mu J(\mu) \approx \sum_{k=1}^m \left(\sum_{t=1}^\tau (B_\tau^k - \bar{B}_\tau)^2 \mathds{1}[t = \tau] \nabla_\mu \log N(\theta_k ; \mu, \sigma^2) \right)
\end{equation}
where $B^k_\tau$ is the BC of the $k$th individual, and $\bar{B}$ is the mean BC among all $m$ individuals.
A hypothetical policy gradients version of Evolvability ES could estimate its gradient as follows:
\begin{equation}
\nabla_\mu J(\mu) \approx \sum_{k=1}^m \left(\sum_{t=1}^\tau (B_\tau^k - \bar{B}_\tau)^2 \mathds{1}[t = \tau] \nabla_\mu \log \pi(a_t | s_t, \mu + \epsilon) \right)
\end{equation}
The formulation above in terms of an explicit sum over $\tau$ timesteps and an indicator on the current timestep makes it easy to extend both of these algorithms to allow for diversity of behavior at multiple timesteps, closer to a traditional RL per-timestep reward function.

\section{Theoretical Results} \label{sec:theory}
\begin{theorem}
    For every pair of continuous functions $g : [0,1]^n \to \mathbb{R}$ and $h : [0,1]^n \to \mathbb{R}$, and for every $\epsilon>0$, there exists a $\delta_2>0$ such that for every $0 < \delta_1 < \delta_2$, there exists $5$-hidden-layer neural network $f(x ; W^1, \ldots, W^5) : [0,1]^n \to \mathbb{R}$ with ReLU activations such that for any distribution of weight perturbations such that each component is sampled i.i.d.\ from a distribution with distribution function $F$, $|f(x; \widetilde{W}^1, \ldots, \widetilde{W}^5) - g(x)| < \epsilon$ for all $x \in [0,1]^n$ with probability at least $F(\delta_2) - F(\delta_1)$, and $|f(x; \widetilde{W}^1, \ldots, \widetilde{W}^5) - h(x)| < \epsilon$ for all $x \in [0,1]^n$ with probability at least $F(-\delta_1) - F(-\delta_2)$, where $\widetilde{W}^l$ are the perturbed weights.
\end{theorem}
\begin{proof}[Proof sketch]
    First, note that $g$ and $h$ can each be arbitrarily approximated as 2-layer neural networks $G$ and $H$.
    
    Next, construct a 2-layer neural network $\sigma : \mathbb{R} \to \mathbb{R}^2$ with ReLU activations which is uniformly $\delta$-close in each component to the function $x \mapsto (I[x>0], I[x<0])$.
    
    Each of these $3$ neural networks is uniformly continuous both in the inputs and in the weights). We now connect them into a neural network $F$. Create a node $K$ in the first hidden layer of $F$ with bias $B$, and connect $G$ and $H$ to the inputs such that their outputs are in the second hidden layer of $F$. Also construct a node $K'$ in the second hidden layer of $F$, and label the weight connecting $K$ and $K'$ $w_k$. Now place $\sigma$ such that its outputs are in the fourth hidden layer of $F$.
    Next, create two nodes $\hat{G}$ and $\hat{H}$ in the $5$th hidden layer of $F$, initialize the weight connecting $G$ and $\hat{G}$ to $1$, and that connecting $H$ and $\hat{H}$ to $1$. Label the weight connecting $\sigma_1$ to $\hat{G}$ $w_g$ and that connecting $\sigma_2$ to $\hat{H}$ $w_h$. Finally, connect $\hat{G}$ and $\hat{H}$ to the output node $F$ with weight $1$ each. Initialize all unmentioned weights to $0$.
    
    Now pick $\delta_2$ such that uniformly over perturbations less than $\delta/2$, nodes $G$ and $H$ are $\epsilon_1$-close to their original outputs, and such that nodes $\sigma_1$ and $\sigma_2$ are $\epsilon_1/M$-close to its original output, where $M$ is a bound on the activations of nodes $g$ and $h$. 
    
    Fix $0 < \delta_1 < \delta_2$. Set $B$ such that every perturbation in $(\delta_1, \delta_2)$ to $w_k$ results in node $K$ positive over all inputs, and every perturbation in $(-\delta_2, -\delta_1)$ to $w_k$ results in node $K$ negative over all inputs. We can do this because the inputs to $K$ are bounded.
    
    Then for perturbations to $w_k$ in $(\delta_1, \delta_2)$, we can make $\sigma$ $\epsilon_2$-close to $(1, 0)$ (in each component). We can then set $w_g$ and $w_h$ small enough that when $\sigma_1$ is $\epsilon_1/M$-close to $1$, node $\hat{G}$ is $0$ and node $\hat{H}$ is $\epsilon_3$-close to its original output, and when $\sigma_2$ is $\epsilon_1/M$-close to $1$, $\hat{H}$ is $0$ and $\hat{G}$ is $\epsilon_3$-close to its original output. Putting everything together, we can choose $\epsilon_1$, $\epsilon_2$, and $\epsilon_3$ such that for perturbations to $w_k$ in $(\delta_1, \delta_2)$, $F$ is $\epsilon$-close to $g$, and for perturbations to $w_k$ in $(-\delta_2, -\delta_1)$, $F$ is $\epsilon$-close to $h$.
\end{proof}
We hope that future work will strengthen these results in several ways. First, the number of layers can likely be made more economical, and the results should be extensible to other activation functions like Tanh. It should also be possible to merge more than two continuous functions in a similar way; a more open-ended question is whether, as in the Ant domain, some uncountable family of functions can be merged into a single neural network.

\end{document}